%% file: main.tex
\documentclass{article} 
\usepackage[dvipsnames]{xcolor}
\usepackage{iclr2025_conference,times}

\usepackage{hyperref}
\usepackage{url}

\usepackage{wrapfig}
\usepackage[inline]{enumitem}

\usepackage{microtype}
\usepackage{graphicx}
\usepackage{subfigure}
\usepackage{booktabs} 

\usepackage{nicefrac}
\usepackage{algorithm}
\usepackage{algorithmic}

\usepackage{amsmath}
\usepackage{amssymb}
\usepackage{mathtools}
\usepackage{amsthm}
\usepackage{xspace}

\usepackage[capitalize,noabbrev]{cleveref}

\theoremstyle{plain}
\newtheorem{theorem}{Theorem}[section]

\theoremstyle{definition}

\theoremstyle{remark}

\usepackage[textsize=tiny]{todonotes}

\usepackage{notation}

\title{Optimizing Backward Policies in GFlowNets \\ via Trajectory Likelihood Maximization}

\author{Timofei Gritsaev 
\\
HSE University \\
Constructor University, Bremen \\
\texttt{tgritsaev@gmail.com} \\
\And 
Nikita Morozov
\\
HSE University \\
\texttt{nvmorozov@hse.ru} \\
\And 
Sergey Samsonov
\\
HSE University \\
\texttt{svsamsonov@hse.ru} \\
\AND
Daniil Tiapkin \\
CMAP – CNRS – {\'E}cole polytechnique – Institut Polytechnique de
Paris \\
Université Paris-Saclay, CNRS, Laboratoire de mathématiques d'Orsay  \\
\texttt{daniil.tiapkin@polytechnique.edu} \\
}

\iclrfinalcopy 
\begin{document}

\maketitle

\vspace{-0.4cm}
\begin{abstract}
\vspace{-0.1cm}
Generative Flow Networks (GFlowNets) are a family of generative models that learn to sample objects with probabilities proportional to a given reward function. The key concept behind GFlowNets is the use of two stochastic policies: a forward policy, which incrementally constructs compositional objects, and a backward policy, which sequentially deconstructs them. Recent results show a close relationship between GFlowNet training and entropy-regularized reinforcement learning (RL) problems with a particular reward design. However, this connection applies only in the setting of a fixed backward policy, which might be a significant limitation. As a remedy to this problem, we introduce a simple backward policy optimization algorithm that involves direct maximization of the value function in an entropy-regularized Markov Decision Process (MDP) over intermediate rewards. We provide an extensive experimental evaluation of the proposed approach across various benchmarks in combination with both RL and GFlowNet algorithms and demonstrate its faster convergence and mode discovery in complex environments.
\end{abstract}


\input{src/intro}

\input{src/background}

\input{src/method}

\input{src/experiments}

\input{src/conclusion}

\newpage

\subsubsection*{Acknowledgments}
The authors are grateful to Dmitry Vetrov for valuable
discussions and feedback. The paper was prepared within the framework of the HSE University Basic Research Program. This research was supported in part through computational resources of HPC facilities at HSE University~\citep{kostenetskiy2021hpc}. The work of Daniil Tiapkin was supported by the Paris Île-de-France Région in the framework of DIM AI4IDF.

\bibliography{iclr2025_conference}
\bibliographystyle{iclr2025_conference}

\newpage 

\appendix
\section{Appendix}
\input{app/proof}
\input{app/ablation}
\input{app/experimental_details}

\input{app/plots}

\end{document}

%% file: src/intro.tex
\vspace{-0.3cm}
\section{Introduction}
\label{sec:intro}

\vspace{-0.15cm}

Generative Flow Networks (GFlowNets, \citealp{bengio2021flow}) are models designed to sample compositional discrete objects, e.g., graphs, from distributions defined by unnormalized probability mass functions. They operate by constructing an object through a sequence of stochastic transitions defined by a \textit{forward policy}. This policy is trained to match the marginal distribution over constructed objects with the target distribution of interest. Since this marginal distribution is generally intractable, an auxiliary \textit{backward policy} is introduced, and the problem is reduced to one of matching distributions over complete trajectories, bearing similarities with variational inference~\citep{malkin2022gflownets}.

GFlowNets have found success in various areas, such as biological sequence design \citep{jain2022biological}, material discovery~\citep{hernandez2023crystal}, molecular optimization~\citep{zhu2024sample}, recommender systems~\citep{liu2024modeling}, large language model (LLM) and diffusion model fine-tuning \citep{hu2023amortizing, venkatraman2024amortizing, uehara2024understanding, zhang2024improving}, neural architecture search \citep{chen2023order}, combinatorial optimization \citep{zhang2023solving}, and causal discovery \citep{atanackovic2024dyngfn}.

Theoretical foundations of GFlowNets have been established in the seminal works of \cite{bengio2021flow,bengio2021gflownet}. Most of the literature has since focused on practical applications of these models, leaving their theoretical properties largely unexplored, with a few exceptions~\citep{krichel2024generalization, silva2024analyzing}. However, a recent line of works has brought attention to connections between GFlowNets and reinforcement learning~\citep{tiapkin2024generative, mohammadpour2024maximum, deleu2024discrete, he2024rectifying}, showing that the GFlowNet learning problem is equivalent to a specific RL problem with entropy regularization (also called soft RL, \citealp{neu2017unified, geist2019theory}). This opened a new perspective for understanding GFlowNets. The importance of these findings is supported by empirical evidence, as various RL algorithms have proven useful for improving GFlowNets~\citep{tiapkin2024generative, morozov2024improving,  lau2024qgfn}.

However, these connections still have a limitation related to GFlowNet backward policies. While GFlowNets can be trained with a fixed backward policy, standard GFlowNet algorithms allow the training of the backward policy together with the forward policy~\citep{bengio2021gflownet,malkin2022trajectory, madan2023learning}, resulting in faster convergence of the optimization process. Other algorithms for optimizing backward policies have been proposed in the literature as well~\citep{mohammadpour2024maximum, jang2024pessimistic}, showing benefits for GFlowNet performance. The theory connecting GFlowNets and entropy-regularized RL is based on using the backward policy to add a "correction" to GFlowNet rewards and shows the equivalence between two problems only when the backward policy is fixed. Thus, understanding the backward policy optimization remains a missing piece of this puzzle. Moreover, \cite{tiapkin2024generative} demonstrated that this theoretical gap has practical relevance, as optimizing the backward policy using the same RL objective as the forward policy can either fail to improve or even slow down convergence, highlighting the need for a more refined approach.

In this study, we introduce \textit{the trajectory likelihood maximization} (\TLM) approach for backward policy optimization, which can be integrated with any existing GFlowNet method, including entropy-regularized RL approaches.

We first formulate the GFlowNet training problem as a unified objective involving both forward and backward policies. We then propose an alternating minimization procedure consisting of two steps: (1) maximizing the backward policy likelihood of trajectories sampled from the forward policy and (2) optimizing the forward policy within an entropy-regularized Markov decision process that corresponds to the updated backward policy. The latter step can be achieved by any existing GFlowNet or soft RL algorithm, as it was outlined by \cite{deleu2024discrete}. By approximating these two steps through a single stochastic gradient update, we derive an adaptive approach for combining backward policy optimization with any GFlowNet method, \textit{including soft RL methods.}

Our main contributions are as follows:
\vspace{-4pt}
\begin{itemize}[itemsep=0pt,leftmargin=16pt]
    \item We derive the trajectory likelihood maximization (\TLM) method for backward policy optimization.
    \item The proposed method represents the first unified approach for adaptive backward policy optimization in soft RL-based GFlowNet methods. The method is easy to implement and can be integrated with any existing GFlowNet training algorithm.
    \item We provide an extensive experimental evaluation of \TLM in four tasks, confirming the findings of \cite{mohammadpour2024maximum}, which emphasize the benefits of training the backward policy in a complex environment with less structure.
\end{itemize}
\vspace{-0.15cm}
Source code: \href{https://github.com/tgritsaev/gflownet-tlm}{github.com/tgritsaev/gflownet-tlm}.
\vspace{-0.2cm}

%% file: src/background.tex
\vspace{-0.1cm}
\section{Background}
\label{sec:background}
\vspace{-0.2cm}

\subsection{GFlowNets}
\vspace{-0.1cm}
We aim at sampling from a probability distribution over a finite discrete space \(\mathcal{X}\) that is given as an unnormalized probability mass function \(\cR \colon \mathcal{X} \to \mathbb{R}_{\geq 0}\), which we call the \textit{GFlowNet reward}. We denote \(\mathrm{Z} = \sum_{x \in \mathcal{X}} \cR(x)\) to be an (unknown) normalizing constant.

To formally define a generation process in GFlowNets, we introduce a directed acyclic graph (DAG) \(\mathcal{G} = (\mathcal{S}, \mathcal{E})\), where \(\mathcal{S}\) is a state space and \(\mathcal{E} \subseteq \mathcal{S} \times \mathcal{S}\) is a set of edges (or transitions). There is exactly one state, $s_0$, with no incoming edges, which we refer to as the \textit{initial state}. All other states can be reached from $s_0$, and the set of \textit{terminal states} with no outgoing edges coincides with the space of interest $\mathcal{X}$. Non-terminal states $s \notin \cX$ correspond to ``incomplete'' objects and edges $s \to s'$ represent adding "new components" to such objects, transforming $s$ into $s'$. Let \(\mathcal{T}\) denote the set of all complete trajectories \(\tau = \left(s_0, s_1, \ldots, s_{n_{\tau}}\right)\) in the graph, where \(\tau\) is a sequence of states such that \((s_i \to s_{i + 1}) \in \mathcal{E}\) and that starts at \(s_0\) and finishes at some terminal state \(s_{n_{\tau}} \in \mathcal{X}\). As a result, any complete trajectory can be viewed as a sequence of actions that constructs the object corresponding to $s_{n_{\tau}}$ starting from the "empty object" $s_0$.

We say that a state $s'$ is a child of a state $s$ if there is an edge $(s \to s') \in \mathcal{E}$. In this case, we also say that $s$ is a parent of $s'$. Next, for any state $s$, we introduce the \textit{forward policy}, denoted by $\PF(s'|s)$ for $(s \to s') \in \cE$, as an arbitrary probability distribution over the set of children of the state $s$. In a similar fashion, we define the \textit{backward policy} as an arbitrary probability distribution over the parents of a state $s$ and denote it as $\PB(s'|s)$, where $(s' \to s) \in \cE$.

Given these two definitions, the main goal of GFlowNet training is a search for a pair of policies such that the induced distributions over complete trajectories in the forward and backward directions coincide:
\begin{equation}
\label{eq:tb}
     \prod_{t=1}^{n_\tau} \PF \left(s_{t} \mid s_{t-1}\right) = \frac{\cR(s_{n_\tau})}{\rmZ} \prod_{t=1}^{n_\tau} \PB \left(s_{t-1} \mid s_{t}\right)\,, \quad \forall \tau \in \cT\,.
\end{equation}
The relation \eqref{eq:tb} is known as the \textit{trajectory balance constraint} \citep{malkin2022trajectory}. We refer to the left and right-hand sides of \eqref{eq:tb} as to the forward and backward trajectory distributions and denote them as 
\begin{equation}
\label{eq:forward_backward_traj}
\Ptraj{\PF}(\tau) := \prod_{i=1}^{n_{\tau}} \PF(s_i | s_{i-1})\,, \qquad \Ptraj{\PB}(\tau) := \frac{\cR(s_{n_{\tau}})}{\rmZ} \cdot \prod_{i=1}^{n_\tau} \PB(s_{i-1} | s_i)\eqsp,
\end{equation}
where $\tau = (s_0,s_1,\ldots,s_{n_{\tau}}) \in \cT$. If the condition \eqref{eq:tb} is satisfied for all complete trajectories, sampling a trajectory in the forward direction using \(\PF\) will result in a terminal state being sampled with probability \(\cR(x)/\mathrm{Z}\). We will call such $\PF$ a \textit{proper} GFlowNet forward policy.

In practice, we train a model (usually a neural network) that parameterizes the forward policy (and possibly other auxiliary functions) to minimize an objective function that enforces the constraint \eqref{eq:tb} or its equivalent. The main existing objectives are \textit{Trajectory Balance} (\texttt{TB}, \citealp{malkin2022trajectory}), \textit{Detailed Balance} (\texttt{DB}, \citealp{bengio2021gflownet}) and \textit{Subtrajectory Balance} (\texttt{SubTB}, \citealp{madan2023learning}). The \texttt{SubTB} objective is defined as
\begin{equation}\label{eq:SubTB_loss}
\textstyle
\cL_{\mathrm{SubTB}}(\theta; \tau) = \sum\limits_{0 \le j < k \le n_{\tau}} w_{jk}\left(\log \frac{F_\theta(s_j)\prod_{t=j + 1}^{k} \PF(s_t | s_{t - 1}, \theta)}{F_\theta(s_k) \prod_{t=j + 1}^{k} \PB(s_{t - 1} | s_{t}, \theta)} \right)^2\,,
\end{equation}
where $F_{\theta}(s)$ is a neural network that approximates the \textit{flow} function of the state $s$, see \citep{bengio2021gflownet,madan2023learning} for more details on the flow-based formalization of the GFlowNet problem. Here $F_{\theta}(s)$ is substituted with $\cR(s)$ for terminal states $s$, and $w_{jk}$ is usually taken to be $\lambda^{k - j}$ and then normalized to sum to $1$. \texttt{TB} and \texttt{DB} objectives can be viewed as special cases of \eqref{eq:SubTB_loss}, which are obtained by only taking the term corresponding to the full trajectory or to individual transitions, respectively. All objectives allow either training the model in an on-policy regime using the trajectories sampled from $\PF$ or in an off-policy mode using the replay buffer or some exploration techniques. In addition, it is possible to either optimize $\PB$ along with $\PF$ or to use a fixed $\PB$, e.g., the uniform distribution over parents of each state. One can show that given any fixed $\PB$, there exists a unique $\PF$ that satisfies \eqref{eq:tb}; see, e.g., \citep{malkin2022trajectory}.

\vspace{-0.1cm}
\subsection{GFlowNets as Soft RL}
\label{sec:gfn_as_rl}
\vspace{-0.1cm}

In reinforcement learning~\citep{sutton2018reinforcement}, a typical performance measure of an agent is a \textit{value function}, which is defined as an expected discounted sum of rewards when acting via a given policy. Entropy-regularized reinforcement learning (RL), also known as soft RL (\citealt{neu2017unified, geist2019theory, haarnoja2017reinforcement}), incorporates Shannon entropy $\cH(\PF(\cdot\mid s)) \triangleq \E_{s' \sim \PF(\cdot\mid s)}[-\log \PF(s'\mid s)]$ into the value function. This addition encourages the optimal policy to be more exploratory:
\begin{equation}\label{eq:regularized_value_def}
    \textstyle V^{\PF}_{\lambda}(s; r) \triangleq \E_{\PF}\bigg[ \sum\limits_{t=0}^\infty \gamma^t \big(r(s_t, s_{t+1}) + \lambda \cH(\PF(\cdot \mid s_t))\big) \,\Big|\, s_0 = s \bigg],
\end{equation}
where $\lambda \geq 0$ is the regularization coefficient. Note that we use the next state instead of the more conventional action representation, as there is a one-to-one correspondence between the action taken and the resulting next state in DAG environments. Similarly, the regularized Q-values $Q^{\PF}_{\lambda}(s, s')$ are defined as the expected discounted sum of rewards, augmented by Shannon entropy, given an initial state $s_0 = s$ and the next state $s_1 = s'$. The regularized optimal policy $\cP^{\star}_{\mathrm{F}, \lambda}$ is the policy that maximizes $V^{\PF}_{\lambda}(s)$ for any state $s$. \textbf{Note:} in standard RL notation, a policy is typically denoted as $\pi$. Here, we use $\PF$, consistently with GFlowNet notation, to avoid notational clutter.

It was proven by~\cite{tiapkin2024generative} that training a GFlowNet policy $\PF$ with a fixed $\PB$ can be formulated as a soft RL problem in a deterministic MDP represented by a directed acyclic graph $\mathcal{G}$, where actions correspond to transitions over edges. For transitions $(s \to x)$ that lead to terminal states, the RL rewards are defined as $r^{\PB}(s, x) = \log \PB(s \mid x) + \log \cR(x)$, while for intermediate transitions $(s \to s')$, the rewards are $r^{\PB}(s, s') = \log \PB(s \mid s')$. By setting $\lambda = 1$ and $\gamma = 1$, it can be shown that the optimal policy $\cP^{\star}_{\mathrm{F}, \lambda = 1}(\cdot \mid s)$ in this regularized MDP coincides with the proper GFlowNet forward policy $\PF(\cdot \mid s)$, which is uniquely determined by $\PB$ and $\cR$ (Theorem~1, \citealp{tiapkin2024generative}), thereby inducing the correct marginal distribution over terminal states $\cR(x)/\rmZ$.

In addition, Proposition~1 of~\cite{tiapkin2024generative} provides a connection between the corresponding regularized value function at the initial state $s_0$ for any forward policy $\PF$ (not necessarily proper) and KL-divergence between the induced trajectory distributions:
\[
\textstyle
V^{\PF}_{\lambda = 1}(s_0; r^{\PB}) = \log \rmZ - \KL(\Ptraj{\PF} \Vert \Ptraj{\PB})\,.
\]
The main practical corollary of these results is that any RL algorithm that works with entropy regularization can be utilized to train GFlowNets when $\PB$ is fixed. For example, \cite{tiapkin2024generative} demonstrated the efficiency of the classical \texttt{SoftDQN} algorithm \citep{haarnoja2017reinforcement} and its modified variant called \texttt{MunchausenDQN} \citep{vieillard2020munchausen}. Moreover, it turns out that, under this framework, the existing GFlowNet training algorithms can be derived from existing RL algorithms. \cite{tiapkin2024generative} showed that on-policy \texttt{TB} corresponds to policy gradient algorithms, and \texttt{DB} corresponds to a dueling variant of \texttt{SoftDQN}. At the same time, \cite{deleu2024discrete} showed that \texttt{TB}, \texttt{DB} and \texttt{SubTB} algorithms can be derived from path consistency learning (\texttt{PCL}, \citealp{nachum2017bridging}) under the assumption of a fixed backward policy.

\vspace{-0.1cm}
\subsection{Backward Policies in GFlowNets}
\vspace{-0.1cm}
The idea of backward policy optimization is essential to understanding the GFlowNets training procedure. In particular, the most straightforward approach used in GFlowNet literature \citep{malkin2022trajectory,bengio2021gflownet} proposes to optimize the forward and backward policies directly through the same GFlowNet objective (e.g.,  \eqref{eq:SubTB_loss}). This approach can accelerate the speed of convergence~\cite{malkin2022trajectory}, at the same time potentially leading to less stable training~\citep{zhang2022generative}. This phenomenon motivates studying the backward policy optimization in the recent works of \cite{mohammadpour2024maximum} and \cite{jang2024pessimistic}. 
\par 
\cite{mohammadpour2024maximum} suggested using the backward policy with maximum possible trajectory entropy, thus focusing on the \emph{exploration} challenges of GFlowNets. Such policy is proven to be $\PB(s \mid s') = n(s)/n(s')$, where $n(s)$ is the number of trajectories which starts at $s_0$ and end at $s$. It corresponds to the uniform one if the number of paths to all parent nodes is equal. When $n(s)$ cannot be computed analytically, \cite{mohammadpour2024maximum} propose to learn $\log{n(s)}$, $s \in \mathcal{S}$ alongside the forward policy using its relation to the value function of the soft Bellman equation in the \emph{inverted MDP} (see Definition 2 in \cite{mohammadpour2024maximum}). \cite{mohammadpour2024maximum} utilize RL as a tool to find the maximum entropy backward policy and make the connection to RL solely for such policies. In contrast, our work theoretically considers the simultaneous optimization of the forward and the backward policy from the RL perspective, and develops an optimization algorithm grounded in it. The approach of \cite{mohammadpour2024maximum} showed consistently better results in less structured tasks, like QM9 molecule generation (see \Cref{sec:molecule_description} for a detailed description). At the same time, more structured environments with a less challenging exploration counterpart show less benefits of the proposed backward training approach. 
\par 
In another line of work, \cite{jang2024pessimistic} claim that the existing GFlowNets training procedures tend to under-exploit the high-reward objects and propose a Pessimistic Backward Policy approach. Thus, the primary aim of \cite{jang2024pessimistic} is to focus on the \emph{exploitation} of the current information about high-reward trajectories. Towards this aim, they focus on maximizing the observed backward flow $\Ptraj{\PB}(\tau)$ (see \eqref{eq:forward_backward_traj}) for trajectories leading to high-reward objects, which are stored in a replay buffer. {Our approach shares similarities with this method as both compute the loss over whole trajectories but differs in theoretical motivation and the choice of trajectories for $\PB$ optimization.}
\par 
As a limitation of both \cite{mohammadpour2024maximum} and \cite{jang2024pessimistic}, we mention the fact that only a single GFlowNet training objective is used in both papers (\texttt{SubTB} and \texttt{TB} respectively) to evaluate approaches for backward policy optimization, while we carry out experimental evaluation with various GFlowNet training objectives (see Section~\ref{sec:experiments}).

%% file: src/method.tex
\vspace{-0.1cm}
\section{Trajectory Likelihood Maximization}
\vspace{-0.1cm}
\label{sec:method}
The objective of our method is to formalize the optimization process for the backward policy for reinforcement learning-based approaches. It is worth mentioning that soft RL methods cannot address the changing of the reward function, except for reward shaping schemes~\citep{ng1999policy} that preserve the total reward of any trajectory. Therefore, we need to return to the underlying GFlowNet problem. Let us look at the following optimization problem:
\begin{equation}
\label{eq:minimize_joint_kl}
\textstyle 
\min_{\PF \in \PiF, \PB \in \PiB} \KL(\Ptraj{\PF} \Vert \Ptraj{\PB}) \,,
\end{equation}
where $\PiF$ and $\PiB$ represent the spaces of forward and backward policies, respectively, and $\Ptraj{\PF}$ and $\Ptraj{\PB}$ are defined in \eqref{eq:forward_backward_traj}. It is easy to see that any solution to the problem \eqref{eq:minimize_joint_kl} satisfies the trajectory balance constraint \eqref{eq:tb} and thus induces a proper GFlowNet forward policy. Additionally, it is worth mentioning that for any fixed $\PB$ the problem \eqref{eq:minimize_joint_kl} is equivalent to maximizing value 
\[
\textstyle 
V^{\PF}_{\lambda=1}(s_0; r^{\PB}) = \log \rmZ - \KL(\Ptraj{\PF} \Vert \Ptraj{\PB})
\]
over $\PF \in \PiF$ in an entropy-regularized MDP, as established in \citet[Proposition~1]{tiapkin2024generative}. Thus, the joint optimization resembles the RL formulation with non-stationary rewards. To leverage the problem's block structure, we propose a meta-algorithm consisting of two iterative steps, repeated until convergence:
\begin{equation}\label{eq:two_stage_method}
    \textstyle 
    \PB^{t+1} \approx \arg\min_{\PB} \KL\left(\Ptraj{\PF^t} \big\Vert \Ptraj{\PB}\right)\,, \qquad
    \PF^{t+1} \approx \arg\min_{\PF} \KL\left(\Ptraj{\PF} \big\Vert \Ptraj{\PB^{t+1}}\right)\,.
\end{equation}
It is worth noting that if these optimization problems are solved exactly, the algorithm converges after the first iteration. This occurs because, for every fixed backward policy $\PB$, there is a unique forward policy $\PF$, such that $\Ptraj{\PF} = \Ptraj{\PB}$, see, e.g., \citep{malkin2022trajectory}, ensuring that the loss function reaches its global minimum.  In the following sections, we provide implementation details on approximating these two steps.

\paragraph{First Step: Trajectory Likelihood Maximization. } Using the connection between forward KL divergence minimization and maximum likelihood estimation (MLE), we formulate the following \textit{trajectory likelihood maximization} objective:
\begin{equation}
\label{eq:tlm_step}
\textstyle 
    \thetaB^{t+1} \approx \arg\min_{\theta} \E_{\tau \sim \PF^t}\left[ \cL_{\mathtt{TLM}}(\theta; \tau) \right], \qquad \cL_{\mathtt{TLM}}(\theta; \tau) := - \sum_{i=1}^{n_{\tau}}\log \PB(s_{i-1} | s_{i}, \theta)\,.
\end{equation}
In this formulation, $\tau = (s_0,s_1,\ldots,s_{n_{\tau}})$ denotes a trajectory generated by the forward policy $\PF^t$. This step seeks to update the backward policy by minimizing the negative log-likelihood of trajectories generated from the forward policy. Additionally, instead of solving \eqref{eq:tlm_step} for exact $\arg\min$ for every $t$, we perform one stochastic gradient update
\[
\textstyle 
\thetaB^{t+1} = \thetaB^t - \gamma \nabla_\theta \cL_{\mathtt{TLM}}(\thetaB^t; \tau)\,.
\]

\paragraph{Second Step: Non-Stationary Soft RL Problem.} To approximate the second step of \eqref{eq:two_stage_method}, we exploit the equivalence between the GFlowNet framework and the entropy-regularized RL problem. This leads to the following expression:
\begin{equation}
\label{eq:rl_step_abstract}
\textstyle
\PF^{t+1} \approx \arg\min_{\PF \in \PiF} \KL\left(\Ptraj{\PF} \big\Vert \Ptraj{\PB^{t+1}}\right) \iff \PF^{t+1} \approx \arg\max_{\PF \in \PiF} V^{\PF}_{\lambda=1}\left(s_0; r^{\PB^{t+1}}\right)\,,
\end{equation}
where $r^{\PB}$ is the RL reward function corresponding to the backward policy $\PB$. This step can be solved using any soft RL method, such as \SoftDQN \citep{haarnoja2018soft}. Additionally, it is noteworthy that all existing GFlowNet algorithms with a fixed backward policy can be viewed as variations of existing RL methods, see, e.g., \citep{deleu2024discrete}. Thus, they can be used to solve the optimization problem in \eqref{eq:rl_step_abstract}.

To mitigate the computational overhead of searching for exact $\arg\min$ in \eqref{eq:rl_step_abstract}, we also propose to perform a single stochastic gradient update in the corresponding GFlowNet training algorithm
\[
    \thetaF^{t+1} = \thetaF^t - \eta \nabla_\theta \cL_{\mathtt{Alg}}(\thetaF^t; \tau, \PB^{t+1})\,,
\]
where $\cL_{\mathtt{Alg}}$ represents the loss function associated with a GFlowNet or soft RL method, such as \SubTB or \SoftDQN. Here, $\tau$ denotes a (possibly off-policy) trajectory. 

The complete procedure can be interpreted as a soft RL method with changing rewards. Our suggested method is summarized in \Cref{algo:tlm} and can be paired with any GFlowNet training method \texttt{Alg} (e.g., \DB, \TB, \SubTB, or \SoftDQN). While for $\PB$ training our approach uses on-policy trajectories, $\PF$ can still be trained off-policy, e.g., by utilizing a replay buffer that stores trajectories or transitions.

\begin{algorithm}[H]
\caption{Trajectory Likelihood Maximization }
\begin{algorithmic}[1]\label{algo:tlm}
   \STATE {\bfseries Input:} Forward and backward parameters $\thetaF^1$, $\thetaB^1$, any GFlowNet loss function $\cL_{\mathtt{Alg}}$,\\ \textit{(optional)} experience replay buffer $\cB$;
   \FOR{$t=1$ {\bfseries to} $N_{\text{iters}}$}
   \STATE Sample a batch of trajectories $\{\tau_k^{(t)}\}_{k=1}^{K}$ from the forward policy $\PF(\cdot | \cdot, \thetaF^t)$; 
   \STATE \textit{(optional)} Update $\B$ with $\{\tau_k^{(t)}\}_{k=1}^{K}$; 
   \STATE Update $\thetaB^{t+1} = \thetaB^{t} - \gamma_{t} \cdot \frac{1}{K} \sum_{k=1}^K \nabla \cL_{\mathtt{TLM}}(\thetaB^{t}; \tau^{(t)}_k)  $, see \eqref{eq:tlm_step};
   \STATE \textit{(optional)} Resample a batch of trajectories $\{\tau_k^{(t)}\}_{k=1}^{K}$ from $\cB$;
   \STATE Update $\thetaF^{t+1} = \thetaF^{t} - \eta_{t} \cdot \frac{1}{K} \sum_{k=1}^K  \nabla \cL_{\mathtt{Alg}}(\thetaF^{t}; \tau^{(t)}_k, \PB(\cdot|\cdot, \thetaB^{t+1}))  $;
   \ENDFOR
\end{algorithmic}
\end{algorithm}

\paragraph{Convergence of the method}

Next, we show why this method indeed solves the GFlowNet learning problem. First, we introduce a \textit{non-stationary soft reinforcement learning problem} of minimizing the worst-case dynamic average regret
\begin{equation}\label{eq:average_regret}
    \textstyle
    \uregret^T := \frac{1}{T} \sum_{t=1}^T V^{\PF^\star}_{\lambda=1}(s_0; r^{t}) - V^{\PF^t}_{\lambda=1}(s_0; r^{t})\,,
\end{equation}
where $\{r^t\}_{t\in[T]}$ is a sequence of reward functions, and $r^t$ is revealed to a learner before selecting a policy $\PF^t$.
Following \cite{zahavy2021reward}, we conjecture that existing RL algorithms are adaptive to the setting of known but non-stationary reward sequences. The implementation of Sampler player of \texttt{EntGame} algorithm by \cite{tiapkin2023fast} is an example of such a regret minimization algorithm. Additionally, we notice that the optimization of dynamic regret is well-studied in the online learning literature, even in a more challenging setting of revealing the corresponding reward function \emph{after} playing a policy \citep{zinkevich2003online,besbes2015non}.

Next, we provide the convergence result for our two-step procedure, using a stability argument for the first step. The proof is given in Appendix~\ref{app:proof}.

\begin{theorem}\label{th:convergence}
    Assume that (1) the backward updates are stable, i.e., $\sup_{t \geq 0} \norm{\PB^T - \PB^{T+t}}_1 \to 0$ as $T \to \infty$, and (2) the forward updates follow a non-stationary regret minimization algorithm, i.e., $\uregret^T \to 0$ as $T \to \infty$. Then, there exists a proper GFlowNet forward policy $\PF^\star$ that induces the marginal distribution $\cR(x)/\rmZ$ over terminal states, such that $\frac{1}{T} \sum_{t=1}^T \Ptraj{\PF^t} \to \Ptraj{\PF^\star}$ as $T \to \infty$.
\end{theorem}
During numerical experiments, we observed that {enforcing stability in backward updates, particularly by using a decaying learning rate and a target network, significantly improves convergence in practice}. Furthermore, as the theorem shows, this stability is essential for theoretical convergence. We discuss stability techniques we utilize alongside our algorithm in Appendix~\ref{app:abl}.


%% file: src/experiments.tex
\vspace{-0.15cm}
\section{Experiments}
\label{sec:experiments}
\vspace{-0.15cm}

\begin{figure}[!t]
    \centering
    \includegraphics[width=1\linewidth]{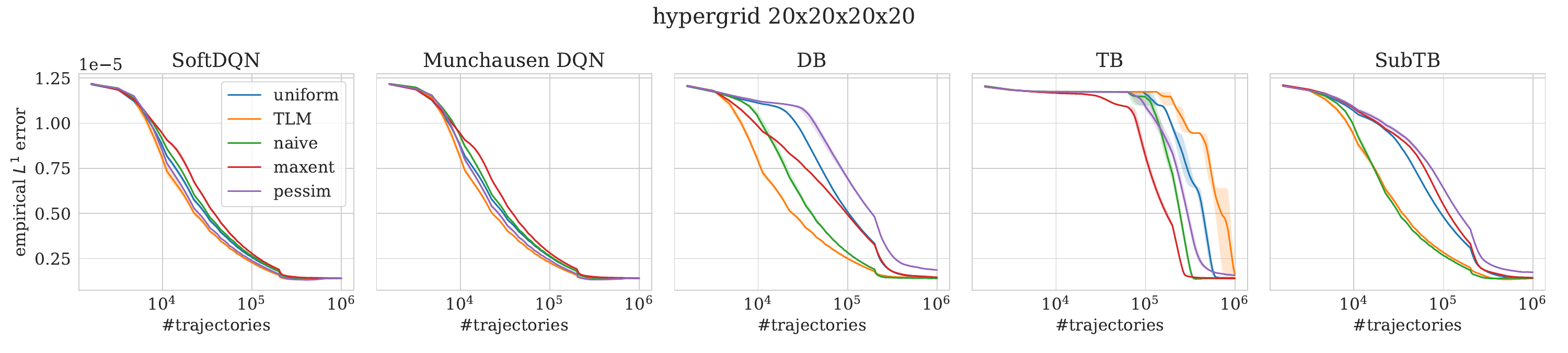}
    \includegraphics[width=1\linewidth]{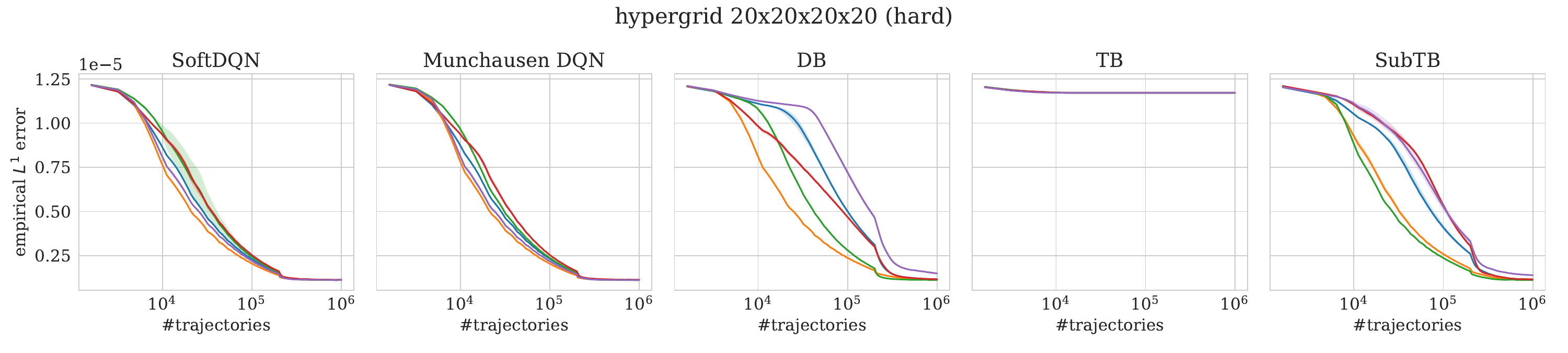}
    \caption{$L^1$ distance between target and empirical sample distributions over the course of training on the standard (\textbf{top row}) and hard (\textbf{bottom row}) hypergrid environments for each method. \emph{Lower values indicate better performance.}} 
    \label{fig:hypergrid}
\end{figure}

We conduct our experimental evaluation on hypergrid~\citep{bengio2021flow} and bit sequence~\citep{malkin2022trajectory} environments, as well as on two molecule design environments: sEH~\citep{bengio2021flow} and QM9~\citep{jain2023multi}. Additional experimental details and hyperparameter choices are provided in Appendix~\ref{app:exp}.

We evaluate four GFlowNet training methods: \MDQN (based on the framework of~\cite{tiapkin2024generative}), \DB~\citep{bengio2021gflownet}, \TB~\citep{malkin2022trajectory}, and \SubTB~\citep{madan2023learning}, which we refer to as \textit{GFlowNet algorithms} (denoted as $\cL_{\mathtt{Alg}}$ in the previous section). On hypergrids, we also include results for \SoftDQN~\citep{tiapkin2024generative}. Alongside these algorithms, we consider five strategies for learning or selecting the backward policy:
\begin{itemize}
    \item our approach (\TLM);
    \item fixed uniform backward policy (\uniform);
    \item simultaneously learning the backward policy with \(\PF\) using the same objective (\naive);
    \item maximum entropy backward policy (\maxent)~\citep{mohammadpour2024maximum};
    \item {pessimistic backward policy} 
 (\pessim)~\citep{jang2024pessimistic}.
\end{itemize}
We refer to these strategies as \textit{backward approaches}. In this section, we denote the distribution induced by a forward policy $\PF$ over terminal states as $P_\theta(x)$ for $x \in \cX$, representing the probability of sampling $x$ from our GFlowNet.

\subsection{Hypergrid}

We start experiments with synthetic hypergrid environments introduced by \citet{bengio2021flow}. These environments are sufficiently small to compute target distribution in the closed form, allowing us to directly examine the convergence of $P_\theta(x)$ to $\cR(x) / \rmZ$.

The environment is a \( d \)-dimensional hypercube with a side length equal to \( H \). The state space is represented as $d$-dimensional vectors $(s_1,\ldots,s_d)^\top \in \{0,\ldots,H-1\}^d$ with the initial state being \( (0, \ldots, 0)^\top \). For each state $(s_1,\ldots,s_{d-1})$, there are at most $d+1$ actions. The first action always corresponds to an exit action that transfers the state to its terminal copy, and the rest of $d$ actions correspond to incrementing one coordinate by $1$ without leaving the grid. The number of terminal states here is $|\cX| = H^d$. There are $2^d$ regions with high rewards near the corners of the grid, while states outside have much lower rewards. The exact expression for the rewards is given in \Cref{app:exp_bit}.

We explore environments with the reward parameters taken from \citet{malkin2022trajectory}, referred to as ``standard case'', and with the reward parameters from \citet{madan2023learning}, referred to as a ``hard case''. In the second case, background rewards are lower, which makes mode exploration more challenging. We conduct experiments on a 4-dimensional hypercube with a side length of 20. As an evaluation metric, we use \(L^1\) distance between the true reward distribution and the empirical distribution of the last \( 2 \cdot 10^5  \) terminal states sampled during training.

Figure~\ref{fig:hypergrid} presents the results. For \SoftDQN, \MDQN, and \DB, \TLM shows the fastest convergence for both ``standard'' and ``hard'' reward designs. For the \SubTB algorithm, \TLM shows similar performance to \naive and outperforms \uniform, \maxent and \pessim. \TB is known to have difficulties in this environment~\citep{madan2023learning}, all approaches fail to converge under the ``hard'' reward design. At the same time, with the ``standard'' one, \maxent backward shows the best convergence. An important note is that our results reproduce the findings of~\cite{tiapkin2024generative}: for \SoftDQN and \MDQN training with \uniform backward converges faster than with \naive algorithm, while \TLM shows stable improvement over \uniform. The results and the ranking of algorithms are almost the same for \SoftDQN and \MDQN, so we leave only \MDQN out of two for further experiments.

\begin{figure}[!t]
    \centering
    \includegraphics[width=1\linewidth]{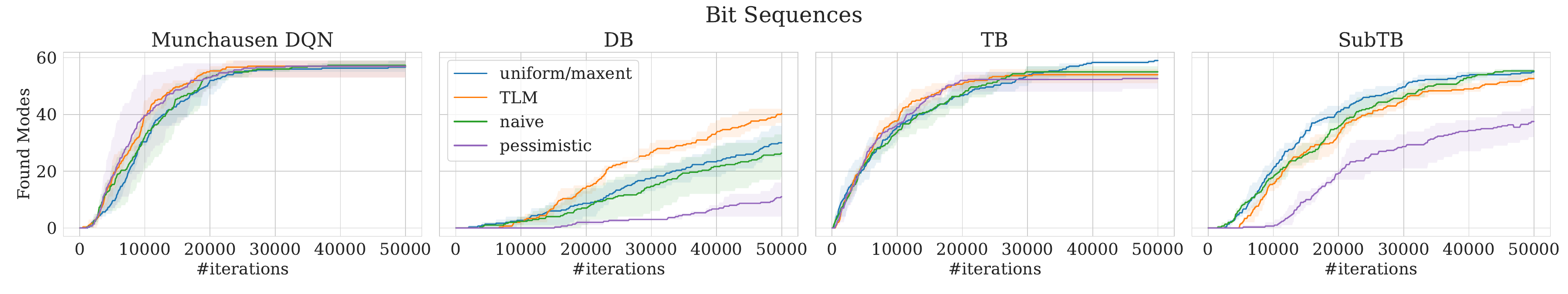}
    \includegraphics[width=1\linewidth]{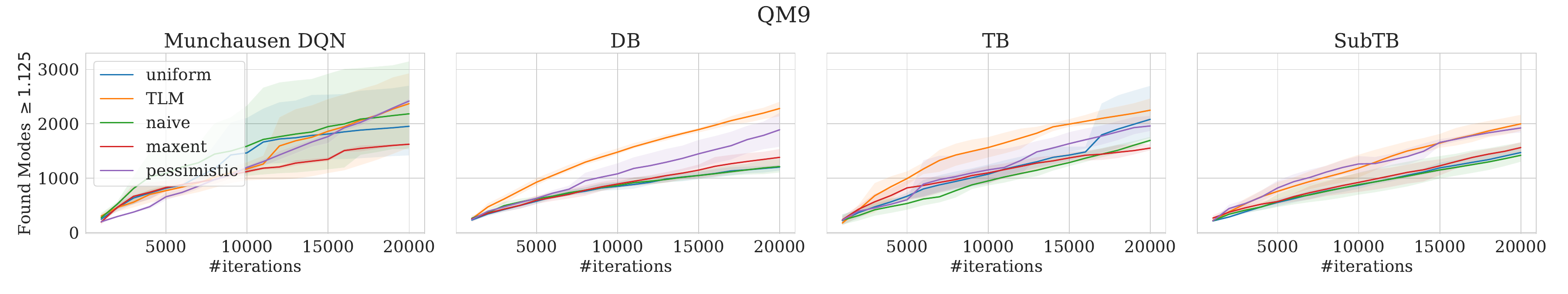}
    \includegraphics[width=1\linewidth]{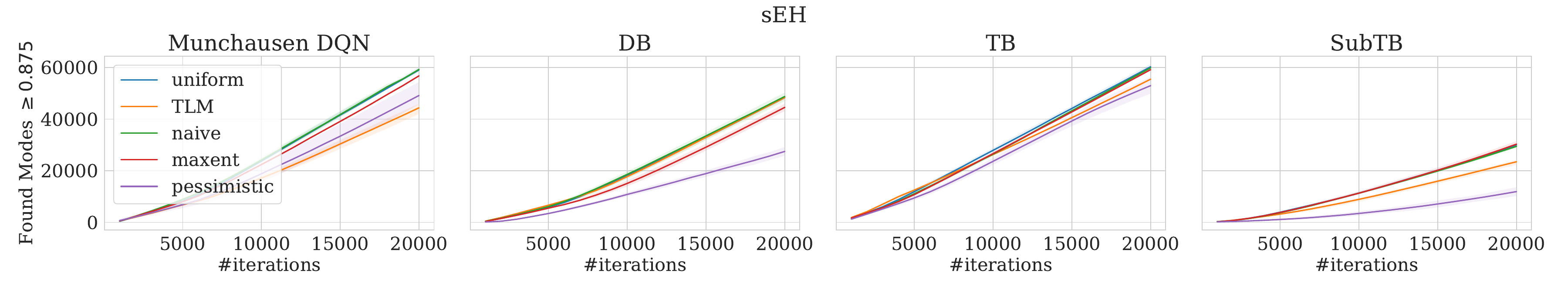}
    \caption{
    \textbf{Top row:} Bit Sequences, the number of discovered modes out of a total of 60 modes for different methods. \textbf{Center row:} QM9, the number of Tanimoto-separated modes with reward higher or equal to $1.125$ for different methods.  \textbf{Bottom row:} sEH, the number of Tanimoto-separated modes with reward higher or equal to $0.875$ for different methods. \emph{Higher values indicate better performance}. {For each pair of a GFlowNet algorithm and a backward approach, the results are presented for the best learning rate chosen in terms of the total number of discovered modes.}} 
    \label{fig:modes_1lr}
\end{figure}

\subsection{Bit Sequences}
In this section, we consider the bit sequence generation task introduced by \citet{malkin2022trajectory}. Following the experimental setup of \cite{tiapkin2024generative}, we modify the state and action spaces to create a non-tree DAG structure, similar to the approach introduced in \citet{zhang2022generative}.

This task is to generate binary sequences of a fixed length $n$, using a vocabulary of $k$-bit blocks.  The state space of this environment corresponds to sequences of $n/k$ words, and each word in these sequences is either an empty word $\oslash$ or one of $2^k$ possible $k$-bit words. The initial state $s_0$ corresponds to a sequence of empty words. The possible actions in each state are replacing an empty word $\oslash$ with one of $2^k$ non-empty words in the vocabulary. The set of terminal states $\cX$ consists of sequences without empty words and corresponds to binary strings of length $n$. The reward function is defined as ${\cR(x) = \exp(-2 \cdot \min_{x' \in \cM} d(x, x'))}$, where $\cM$ is a set of modes and $d$ is Hamming distance. We fix $n = 120$ and $k = 8$ for our experiments. The terminal state space size is $|\cX| = 2^{120}$. Importantly, for this environment, the \uniform backward coincides with \maxent, see Proposition 1 of~\cite{zhang2022generative} and Remark 3 of Theorem 3 of~\cite{mohammadpour2024maximum}. 

To evaluate the performance, we use the same metrics as in \citet{malkin2022trajectory} and \citet{tiapkin2024generative}: the number of modes found during training (number of sequences from $\cM$ for which a terminal state within a distance of $30$ has been sampled) and Spearman correlation on the test set between \( \mathcal{R}(x) \) and an estimate of \( P_\theta \). Since computing the exact probability of sampling a terminal state is intractable due to a large number of paths leading to it, we use a Monte Carlo estimate following the approach of \citet{zhang2022generative}; see also \Cref{app:exp}. {We train all models with various choices of the learning rate, treating it as a hyperparameter, and provide the results depending on its value, similarly to~\citet{madan2023learning}.}

Figure~\ref{fig:modes_1lr} shows the number of modes for different GFlowNet algorithms and backward approaches found over the course of training. We observe that \TLM shows a significant improvement for \DB and a minor one for \MDQN in comparison to other backward approaches, where in the latter case all backward approaches find all 60 modes. \TB and \SubTB also find almost all modes, and \TLM does not affect the results much in comparison to \uniform and \naive, while outperforming \pessim in case of \SubTB. Full plots for modes across varying learning rates are presented in Figure~\ref{fig:modes_bitseq_full} in Appendix. Figure~\ref{fig:corr} (top) presents Spearman correlation between $\cR$ and $P_\theta$ on the test set for the same GFlowNet algorithms and varying learning rates. \TLM shows better or similar performance to the baselines across all GFlowNet algorithms if the optimal learning rate is chosen. Moreover, for \DB and \SubTB, \TLM shows steady improvement over the baselines for all learning rates.

\begin{figure}[!t]
    \centering
    \includegraphics[width=1\linewidth]{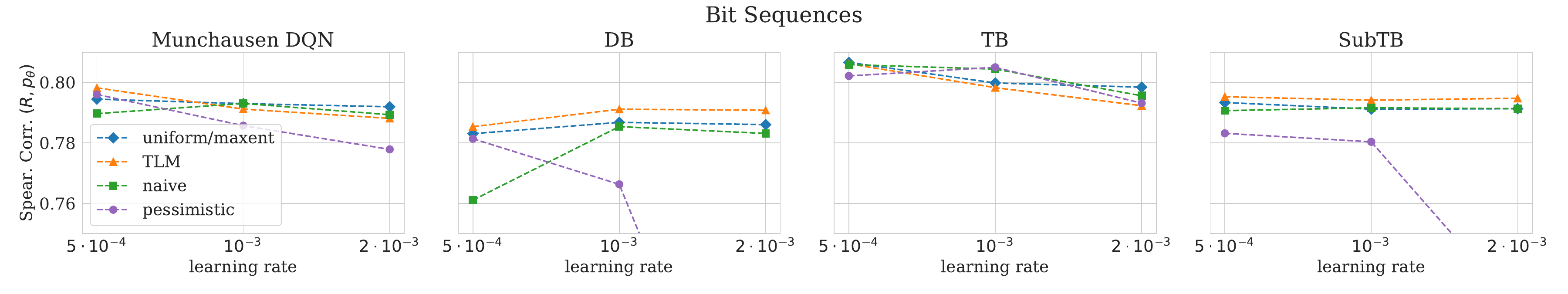}
    \includegraphics[width=1\linewidth]{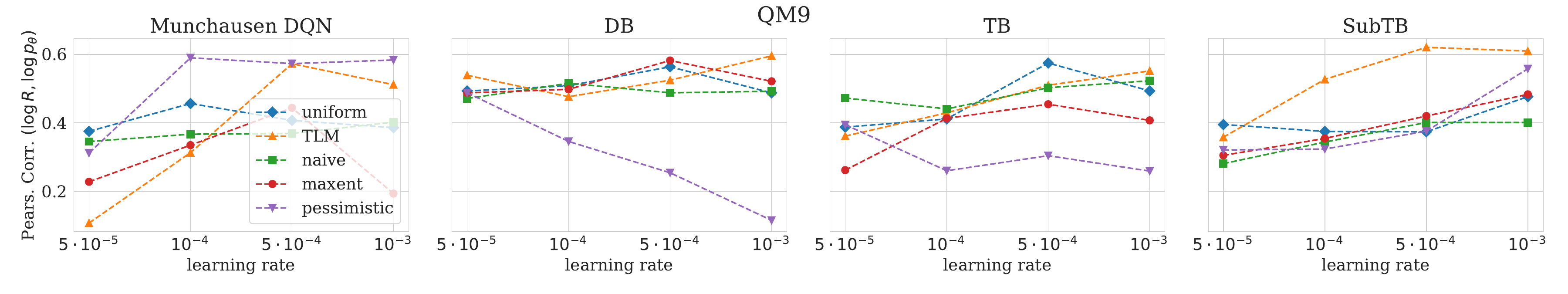}
    \includegraphics[width=1\linewidth]{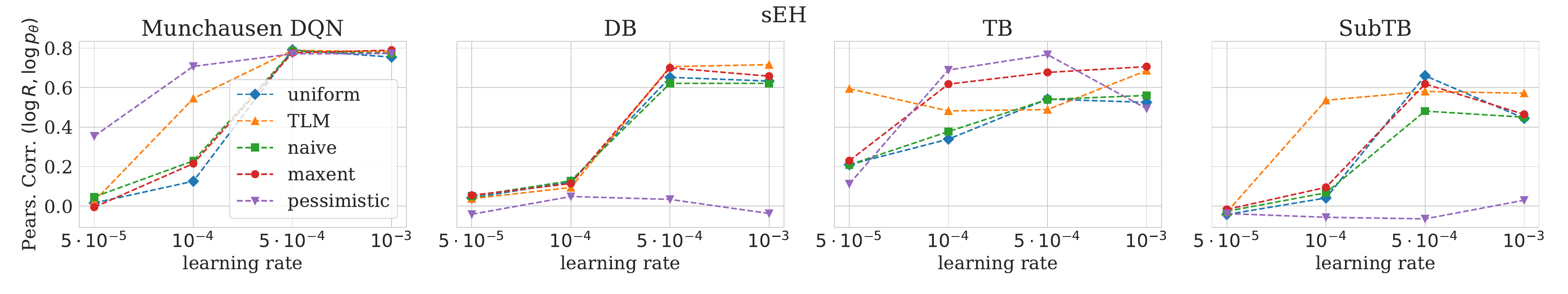}
    \vspace{-0.3cm}
    \caption{\textbf{Top row:} Bit Sequences, Spearman correlation between $\cR$ and $P_\theta$ on a test set for different methods and varying learning rate $\in \{ 5 \cdot 10^{-4}, 10^{-3}, 2 \cdot 10^{-3} \}$. \textbf{Center row:} QM9, Pearson correlation between $\log \cR$ and $\log P_\theta$ on the fixed subset of the QM9 dataset \citep{ramakrishnan2014qm9} for different methods and varying learning rate $\in \{ 5 \cdot 10^{-5}, 10^{-4}, 5 \cdot 10^{-4}, 10^{-3} \}$. \textbf{Bottom row:} sEH, Pearson correlation between $\log \cR$ and $\log P_\theta$ on the test set from \citet{bengio2021flow} for different methods and varying learning rate $\in \{ 5 \cdot 10^{-5}, 10^{-4}, 5 \cdot 10^{-4}, 10^{-3} \}$.\emph{Higher values indicate better performance}. {We note here that \pessim backward policy can be very sensitive to the choice of learning rate.} }
    \label{fig:corr}
\end{figure}

\subsection{Molecule Design, sEH and QM9}
\label{sec:molecule_description}


Our final experiments are carried out on molecule design tasks of sEH \citep{bengio2021flow} and QM9 \citep{jain2023multi}. In both tasks, the goal is to generate molecular graphs, with reward emphasizing some desirable property. For both problems, we use pre-trained reward proxy neural networks. For the sEH task, the model is trained to predict the binding energy of a molecule to a particular protein target (soluble epoxide hydrolase)~\citep{bengio2021flow}. For the QM9 task, the proxy is trained on the QM9 dataset \citep{ramakrishnan2014qm9} to predict the HOMO-LUMO gap~\citep{zhang2020molecularmechanicsdrivengraphneural}.

For the sEH task, we follow the framework proposed by \citet{jin2020junction} and generate molecules using a predefined vocabulary of fragments. We use the same 72 fragments as in the seminal work of~\cite{bengio2021flow}. It is essential to mention that these fragments are explicitly selected for the sEH task to simplify high-quality object generation. The states are represented as trees of fragments. The actions correspond to choosing a new fragment and then choosing an atom to which the fragment will be attached. There is also a special stop action that moves the state to its terminal copy and stops the generation process.

For the QM9 task, molecules are generated atom-by-atom and bond-by-bond. Every state is a
connected graph, and actions either add a new node and
edge or set an attribute on an edge. Thus, the graph-building environment is much more expressive than the fragment-based tree-building environment, but it results in a more complex generation task and can lead to construction of invalid molecules. {It is worth mentioning that there exists a simpler fragment-based version of this environment in GFlowNet literature~\citep{shen2023towards, chen2023order}, while we consider the more complex atom-by-atom setup from~\cite{mohammadpour2024maximum}.}

We use the same evaluation metrics for both tasks as proposed in previous works~\citep{bengio2021flow, madan2023learning, tiapkin2024generative}. We track the number of Tanimoto-separated modes above a certain reward threshold captured over the course of training, and Pearson correlation on the test set between $\log \cR(x)$ and $\log P_{\theta}(x)$. For sEH task we use the same test set as in \citet{bengio2021flow}, and for QM9 we use a subset of the dataset introduced in~\cite{ramakrishnan2014qm9}. {We train all models with various choices of the learning rate, treating it as a hyperparameter, and provide the results depending on its value, similarly to~\citet{madan2023learning}.}

{Figure~\ref{fig:modes_1lr} (center and bottom) shows the number of modes for different GFlowNet algorithms and backward approaches found over the course of training. \TLM speeds up mode discovery on QM9 when utilized alongside \DB and \DB and performs on par with the best baseline backward approaches for \MDQN and \SubTB, but shows similar or worse performance when compared to other backward approaches on sEH. However, we note that on sEH no backward approach shows significant improvement over \uniform in terms of mode discovery. Full plots for modes across varying learning rates are presented in Figure~\ref{fig:modes_mols_full} in Appendix. It is worth noting that on the QM9 task, \TLM shows robust improvement over the baselines in the majority of configurations. Figure~\ref{fig:corr} (center and bottom) shows Pearson correlation between \(\log \cR\) and \(\log P_{\theta}\) estimate measured on the test set for various learning rates. \TLM results in better correlations when paired with \SubTB, shows comaprable results with the best baselines when paired with \DB and \TB, and falls behind \pessim when paired with \MDQN while outperforming other backward approaches. However, as Figure~\ref{fig:corr} indicates, \pessim backward policy can be very sensitive to the choice of the learning rate and the GFlowNet algorithm that is used for forward policy training.}

\vspace{-0.05cm}
\subsection{{Discussion}}
\vspace{-0.05cm}

{
From the plots above, one can see that across all GFlowNet algorithms (forward policy training objectives), \TLM generally shows performance that is better or comparable to other backward approaches. The sole exception is the number of discovered modes in the sEH environment, where \TLM can fall behind other backward approaches. To explain this shortcoming, we hypothesize that \TLM is more beneficial in environments with less structure. This is supported by the major improvements to mode exploration that it obtains on QM9, while sometimes even degrading the same metric on sEH. Indeed, molecules in the sEH task are constructed from the predefined set of blocks, while in the QM9 task, they are created from atoms. This manually predefined set adds a lot of structure into the environment, in which one creates a junction tree from these blocks. In comparison, the process of creating an arbitrary graph of atoms imposes less structure and can even lead to construction of invalid molecules in certain cases. We suppose that the strong methodological bias is a possible reason why it is of little utility to consider non-trivial backward approaches in the sEH task and why the \uniform backward approach often has the best or at least comparable performance according to Figure~\ref{fig:modes_mols_full}. This is exactly the hypothesis that was initially put forward by \cite{mohammadpour2024maximum}, and our results align with it well.}

%% file: src/conclusion.tex
\vspace{-0.25cm}
\section{Conclusion}
\vspace{-0.2cm}
\label{sec:conclusion}

In this work, we propose a new method for backward policy optimization that enhances mode exploration and accelerates convergence in complex GFlowNet environments. \TLM represents the first principled method for learning a backward policy in soft reinforcement learning-based GFlowNet algorithms, such as \SoftDQN and \MDQN. We provide an extensive experimental evaluation, demonstrating benefits of \TLM when it is paired with various forward policy training methods, and analyze its shortcomings, arguing that our results support the hypothesis of~\citet{mohammadpour2024maximum} about benefits of backward policy optimization in environments with less structure.

A promising further work direction is using backward policy for exploration as proposed in \cite{kim2023local} and \cite{he2024looking}. Indeed, the ability to sample trajectories starting from high-reward terminal states via $\PB$ provides an opportunity to improve mode exploration. We expect that combining such methods with \TLM-like approaches will additionally improve their performance.

%% file: app/proof.tex
\subsection{Omitted proofs}\label{app:proof}
\begin{proof}[Proof of Theorem~\ref{th:convergence}]
    From the stability of backward updates, the Cauchy criterion implies that there is $\PB^\star \in \PiB$ such that $\PB^T \to \PB^\star$. At the same time, by the choice or rewards $r^t = r^{\PB^t}$, Proposition~1 by \cite{tiapkin2024generative} and joint convexity of KL-divergence
    \[
        \uregret^T = \frac{1}{T} \sum_{t=1}^T \KL\left(\Ptraj{\PF^t} \big\Vert \Ptraj{\PB^{t}}\right) \geq \KL\left(\frac{1}{T} \sum_{t=1}^T 
 \Ptraj{\PF^t} \bigg\Vert \frac{1}{T}\sum_{t=1}^T \Ptraj{\PB^{t}}\right)\,.
    \]
    Notice that a mapping $\PB \mapsto \Ptraj{\PB}$ is continuous, thus $\Ptraj{\PB^{T}} \to \Ptraj{\PB^\star}$, and, as a result, averages of $\Ptraj{\PB^{t}}$ also converge to $\Ptraj{\PB^\star}$. Finally, applying Pinkser's inequality, we have
    \[
        \textstyle \big\Vert \frac{1}{T} \sum_{t=1}^T 
 \Ptraj{\PF^t} - \Ptraj{\PB^\star} \big\Vert_1 \leq \sqrt{2\uregret^T} +  \norm{\frac{1}{T} \sum_{t=1}^T 
 \Ptraj{\PB^t} - \Ptraj{\PB^\star}}_1\,.
    \]
    The right-hand side of the inequality tends to zero as $T \to +\infty$, thus $\frac{1}{T} \sum_{t=1}^T 
 \Ptraj{\PF^t} \to \Ptraj{\PB^\star}$. Finally, since for any $\PB^\star$ there is $\PF^\star$ such that $\Ptraj{\PB^\star} = \Ptraj{\PF^\star}$, we conclude the statement.
\end{proof}

%% file: app/ablation.tex
\subsection{Stability Techniques}
\label{app:abl}

In this section, we highlight important practical techniques and design choices motivated by Theorem~\ref{th:convergence} that we use alongside our \TLM algorithm to enforce stability into the training of $\PB$. 

First, we found it beneficial to either use a lower learning rate for the backward policy or decay it over the course of training (see the next sections for detailed descriptions). 

Second, akin to how the Deep Q-Network algorithm \citep{mnih2015human} utilizes a target network to estimate the value of the next state, we utilize target networks for the backward policy when calculating the loss for the forward policy. For example, \eqref{eq:SubTB_loss} transforms into
\begin{equation}
\label{eq:SubTB_loss_wtargetnet}
\textstyle
\mathcal{L}_{\mathrm{SubTB}}(\theta; \tau) = \sum\limits_{0 \le j < k \le n_{\tau}} w_{jk}\left(\log \frac{F_\theta(s_j)\prod_{t=j + 1}^{k} \PF(s_t | s_{t - 1}, \theta)}{F_\theta(s_k) \prod_{t=j + 1}^{k} \textcolor{red}{\PB(s_{t - 1} | s_{t}, \bar \theta)}} \right)^2
\end{equation}
where the parameters $\bar \theta$ of $\PB(s_{t - 1} | s_{t}, \bar \theta)$ are updated via exponential moving average (EMA) of the online parameters $\theta$ of $\PB(s_{t - 1} | s_{t}, \theta)$. So the loss for the backward policy $\cL_{\mathtt{TLM}}$ is computed using an online backward policy $\PB(s_{t - 1} | s_{t}, \theta)$, and the loss for the forward policy $\cL_{\mathtt{Alg}}$ is computed using a target backward policy $\PB(s_{t - 1} | s_{t}, \bar \theta)$, which is frozen during the gradient update of $\PF$.

Finally, we find it helpful to initialize $\PB$ to the uniform distribution at the beginning of training, which is done by zero-initialization of the last linear layer weight and bias. 

We ablate the impact of the proposed techniques on QM9, where we try to turn off each of the three separately. Results are presented in Figure~\ref{fig:qm9_ablation}. We observe that using a target model and a lower learning rate is crucial, whereas disabling uniform initialization increases variance and shows slightly worse results. For this experiment, we choose \DB as the base algorithm because \TLM overall shows the most significant impact when applied with it compared to \TB, \SubTB, and \MDQN.

\begin{figure}[t!]
    \centering
    \includegraphics[width=0.9\linewidth]{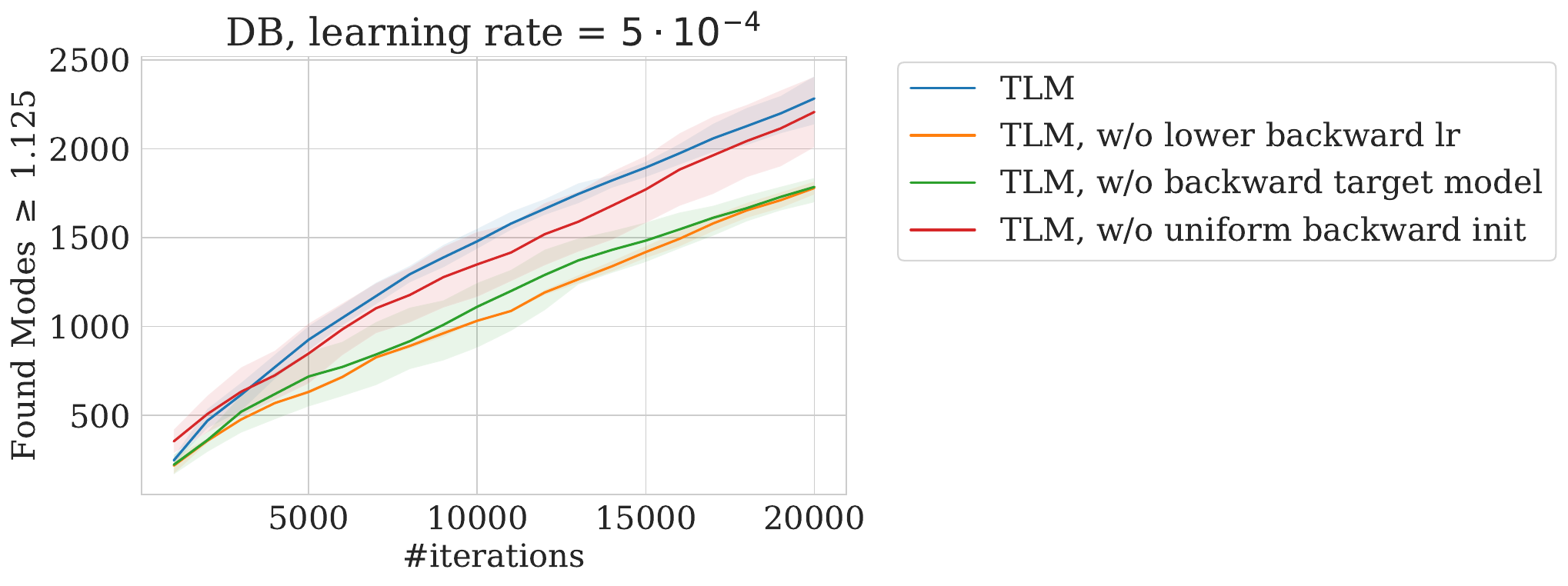}
    \caption{Ablation study of stability techniques on QM9. The number of Tanimoto-separated modes with a reward at least $1.125$ is shown. As a base algorithm, we use \DB with a learning rate of $5 \cdot 10^{-4}$.} 
    \label{fig:qm9_ablation}
\end{figure}

%% file: app/experimental_details.tex
\subsection{Experimental Details}
\label{app:exp}

We utilize PyTorch~\citep{paszke2019pytorch} in our experiments. For hypergrid and bit sequence environments, we base our implementation upon the published code of~\cite{tiapkin2024generative}. For molecule design experiments, our implementations are based on the open source library by Recursion Pharmaceuticals.\footnote{ \href{https://github.com/recursionpharma/gflownet}{\texttt{https://github.com/recursionpharma/gflownet}}} In all our experiments, $\PF$ and $\PB$ share the same neural network backbone, predicting logits via two separate linear heads. For all experiments, mean and std values are computed over three random seeds. 

An important note is that we do one stochastic gradient update over the replay buffer for \pessim backward policy approach~\citep{jang2024pessimistic} at each training iteration to allow for a fair comparison with other methods. \cite{jang2024pessimistic} perform 8 stochastic gradient updates for \pessim $\PB$ at each training iteration in the experiments in their paper, which results in a discrepancy in the number of gradient steps when comparing the approach to other methods.

\vspace{-0.1cm}
\subsubsection{Hypergrid}\label{app:exp_grid}

The reward at a terminal state $s$ with coordinates $(s^1, \ldots, s^D)$ is defined as
\begin{equation*}
\cR(s) = R_0 + R_1 \times \prod_{i = 1}^D \mathbb{I}\left[0.25 < \left|\frac{s^i}{H-1}-0.5\right|\right] + R_2 \times \prod_{i = 1}^D \mathbb{I}\left[0.3 < \left|\frac{s^i}{H-1}-0.5\right| < 0.4\right].
\end{equation*}

\vspace{-0.2cm}
Standard reward uses parameters $(R_0 = 10^{-3}, R_1 = 0.5, R_2 = 2.0)$ and hard reward uses $(R_0 = 10^{-4}, R_1 = 1.0, R_2 = 3.0)$, taken from \cite{bengio2021flow} and \cite{madan2023learning} respectively. 

All models are parameterized using an MLP with 2 hidden layers and 256 hidden units. We use the Adam optimizer with a learning rate of $10^{-3}$ and a batch size of 16 trajectories. For \texttt{SubTB}, we set $\lambda = 0.9$, following \cite{madan2023learning}. For \SoftDQN and \MDQN, we use a prioritized replay buffer~\citep{schaul2016prioritized} and adopt the same hyperparameters as in \cite{tiapkin2024generative}. We sample 256 transitions from the buffer to compute the loss for these methods. 

For the \TLM backward policy, we use the same initial learning rate of $10^{-3}$ with an exponential scheduler, tuning $\gamma$ from \(\{0.999, 0.9999\}\).  The backward policy target network for \TLM uses soft updates with a parameter $\tau=0.25$~\citep{silver2014deterministic}. Since the environment is small, we precompute $\log n(s)$ for all states, allowing us to obtain the \maxent backward exactly. For \pessim backward policy we store trajectories from the last 20 training iterations in the replay buffer and sample 16 trajectories to compute the loss. 

Hypergrid experiments were performed on CPUs. Table~\ref{tab:grid_params} summarizes the chosen hyperparameters.
\input{app/table_hypergrid}

\subsubsection{Bit Sequences}\label{app:exp_bit}

The set of modes \( M \) is defined as in \cite{malkin2022trajectory}, and we choose the same size, \( |M| = 60 \). We set \( H = \{'00000000', '11111111', '11110000', '00001111', '00111100'\} \). Each sequence in \( M \) is generated by randomly selecting \( n/8 \) elements from \( H \) with replacement and then concatenating them. The test set for evaluating reward correlations is generated by taking a mode and flipping \( i \) random bits in it, where this is repeated for every mode and for each \( 0 \le i < n \).

We utilize the same Monte Carlo estimate for \( P_{\theta} \) as presented in \cite{tiapkin2024generative} with $N = 10$:
\[
    P_{\theta}(x) = \mathbb{E}_{\PB(\tau \mid x)}\left[ \frac{\PF(\tau \mid \theta)}{\PB(\tau \mid x)}\right] \approx \frac{1}{N} \sum_{i=1}^{N} \frac{\PF(\tau^i \mid \theta)}{\PB(\tau^i \mid x)}, \quad \tau^i \sim \PB(\tau \mid x).
\]
Notice that any valid $\PB$ can be used here, but for each model, we take the $\PB$ that was fixed/trained alongside the corresponding $\PF$ since such a choice will lead to a lower estimate variance. However, we note that the metric is still very noisy, so for each training run, we compute the metric for all model checkpoints (done every 2000 iterations) and use the maximum value.

All models are parameterized as Transformers \cite{vaswani2017attention} with 3 hidden layers, 8 attention heads, and a hidden dimension of 64. Each model is trained for 50,000 iterations and a batch size of 16 with Adam optimizer. We provide results for learning rates from \(\{ 5 \times 10^{-4}, 10^{-3}, 2 \times 10^{-3} \}\). For \texttt{SubTB} we use \( \lambda = 0.9 \). For \MDQN, we use a prioritized replay buffer~\citep{schaul2016prioritized} and take the same hyperparameters as in \cite{tiapkin2024generative}. To compute the loss, we sample 256 transitions from the buffer for \MDQN. 

For the \TLM backward policy, we use the same initial learning rate as for the forward policy and utilize an exponential scheduler, where $\gamma$ is tuned from \(\{0.9997, 0.9999\}\). The backward policy target network for \TLM uses soft updates with a parameter $\tau=0.1$~\citep{silver2014deterministic}. For \pessim backward policy we store trajectories from the last 20 training iterations in the replay buffer and sample 16 trajectories to compute the loss. 

To closely follow the setting of previous works~\citep{malkin2022trajectory, madan2023learning, tiapkin2024generative}, we use $\varepsilon$-uniform exploration with $\varepsilon = 10^{-3}$. We note that this can introduce a small bias into the gradient estimate of $\nabla_\theta \cL_{\mathtt{TLM}}(\thetaB^t; \tau)$ since $\tau$ will not be sampled exactly from $\PF$.

Each bit sequence experiment was performed on a single NVIDIA V100 GPU. Table~\ref{tab:bit_params} summarizes the chosen hyperparameters.

\input{app/table_bitseq}

\subsubsection{Molecules}\label{app:exp_mol}

For sEH, we use the test set from \citet{bengio2021flow}. For QM9, we select a subset of 773 molecules from the QM9 dataset~\citep{ramakrishnan2014qm9} containing between 3 and 8 atoms. The subset is constructed to ensure an approximately equal representation of different molecule sizes. To compute correlation, we use the same Monte Carlo estimate as in the bit sequence task. It is important to note that \citet{mohammadpour2024maximum} used a different evaluation approach, computing correlation on sampled molecules rather than on a fixed dataset.

In the sEH task, rewards are divided by a constant of 8, and the reward exponent is set to 10. For QM9, we subtract the 95th percentile from all rewards, resulting in the majority of rewards being distributed between 0 and 1, with 5\% of molecules having a reward greater than 1. The reward exponent is also set to 10. We set $\cR(x) = \exp(-75.0)$ for invalid molecules $x$ in QM9. We track the number of Tanimoto-separated modes as described in \citet{bengio2021gflownet}, using a Tanimoto similarity threshold of 0.7. After normalization, the reward thresholds are $0.875$ and $1.125$ for sEH and QM9, respectively.

We use the graph transformer architecture from \cite{jain2023multi} with 8 layers and 256 embedding dimensions for both tasks. Each model is trained for 20,000 iterations using the Adam optimizer. We present results for learning rates in the set \(\{5 \times 10^{-4}, 10^{-4}, 5 \times 10^{-4}, 10^{-3} \}\). The learning rate for $\rmZ$ is fixed at $10^{-3}$ for all experiments. We apply exponential schedulers to all learning rates at each training step and refer readers to the hyperparameter table for the exact decay values. The batch sizes are 256 for sEH and 128 for QM9. For \SubTB, we set $\lambda = 1.0$, following \cite{madan2023learning}. For \MDQN, we do not use a replay buffer, training the model on-policy while using the same other hyperparameters as in \cite{tiapkin2024generative}. 

For \TLM, the learning rate for $\PB$ is initially set to be 10 times smaller than that for $\PF$, while other approaches use the same learning rate for $\PB$ as for $\PF$.  The backward policy target network for \TLM uses soft updates with a parameter $\tau=0.05$~\citep{silver2014deterministic}. To learn $\log n(s)$ for the \maxent backward policy, we follow the approach from \cite{mohammadpour2024maximum}, implemented in the Recursion Pharmaceuticals GFlowNet library. For \pessim backward policy we store trajectories from the last 20 training iterations in the replay buffer and sample 256 trajectories to compute the loss for sEH and 128 for QM9 (the same as batch size for these environments). 

In line with previous works~\citep{malkin2022trajectory, madan2023learning, tiapkin2024generative}, we use $\varepsilon$-uniform exploration with $\varepsilon = 0.05$. To account for the bias introduced into the gradient estimate of $\nabla_\theta \cL_{\mathtt{TLM}}(\thetaB^t; \tau)$, we linearly anneal $\varepsilon$ to zero over the course of training. This annealing approach is applied to all GFlowNet algorithms and backward methods to ensure a fair comparison in terms of the number of discovered modes.

Each molecule generation experiment was conducted on a single NVIDIA A100 GPU. Table~\ref{tab:mol_params} summarizes the chosen hyperparameters.

%% file: app/table_hypergrid.tex
\begin{table}[t!]
\begin{center}
\caption{Hyperparameter choices for hypergrid experiments.}
\label{tab:grid_params}
\vspace{0.3cm}
\begin{tabular}{|l|c|}
\hline
Hyperparatemeter & Value \\
\hline \hline
    Training trajectories & $10^6$ \\ 
    Learning rate & $10^{-3}$ \\
    $Z$ learning rate for \TB & $10^{-1}$  \\
    Adam optimizer $\beta, \epsilon, \lambda$ & $(0.9, 0.999), 10^{-8}, 0$ \\ 
    Number of hidden layers & $2$ \\
    Hidden embedding size & $256$ \\
\hline
    \SubTB $\lambda$ & $0.9$ \\
\hline
    \MDQN $\alpha$ & $0.15$ \\
    \MDQN $\l_0$ &   $-100$ \\
\hline
    Prioritized RB size & $10^5$ \\
    Prioritized RB $\alpha, \beta$ & $0.5, 0.0$ \\
    Prioritized RB batch size & $256$ \\
\hline
    \pessim buffer size & $20 \times 16$ \\
\hline
    \TLM backward learning rate & $10^{-3}$ \\
    \TLM backward lr decay $\gamma$ & $\{ 0.999, 0.9999 \}$ \\ 
   \TLM backward target model $\tau$  & $0.25$ \\
\hline
\end{tabular}
\end{center}
\end{table}

%% file: app/table_bitseq.tex
\begin{table}[t!]
\begin{center}
\caption{Hyperparameter choices for bit sequence experiments.}
\label{tab:bit_params}
\vspace{0.3cm}
\begin{tabular}{|l|c|}
\hline
Hyperparatemeter & Value \\
\hline\hline 
    Training iterations & $5 \times 10^4$ \\ 
    Learning rate & $\{ 5 \times 10^{-4}, 10^{-3}, 2 \times 10^{-3} \}$ \\
    $Z$ learning rate for \TB & $10^{-1}$  \\
    Adam optimizer $\beta, \epsilon, \lambda$ & $(0.9, 0.999), 10^{-8}, 10^{-5}$ \\ 
    $\varepsilon$-uniform exploration & $10^{-3}$ \\
    Number of transformer layers & $3$ \\
    Hidden embedding size & $64$ \\
    Dropout & $0.1$ \\
\hline
    \SubTB $\lambda$ & $0.9$ \\
\hline
    \MDQN $\alpha$ & $0.15$ \\
    \MDQN $\l_0$ &   $-100$ \\
    \MDQN leaf coefficient &  $5$ \\
\hline
    Prioritized RB size & $10^5$ \\
    Prioritized RB $\alpha, \beta$ & $0.9, 0.1$ \\
    Prioritized RB batch size & $256$ \\
\hline
    \pessim buffer size & $20 \times 16$ \\
\hline
    \TLM backward learning rate & learning rate \\
    \TLM backward lr decay $\gamma$ & $ 0.9999$ \\ 
    \TLM backward target model $\tau$ & $0.1$ \\
\hline
\end{tabular}
\end{center}
\end{table}

%% file: app/plots.tex
\input{app/table_mols}

\clearpage

\newpage

\subsection{Full plots}

\begin{figure}[h!]
    \centering
    \includegraphics[width=1\linewidth]{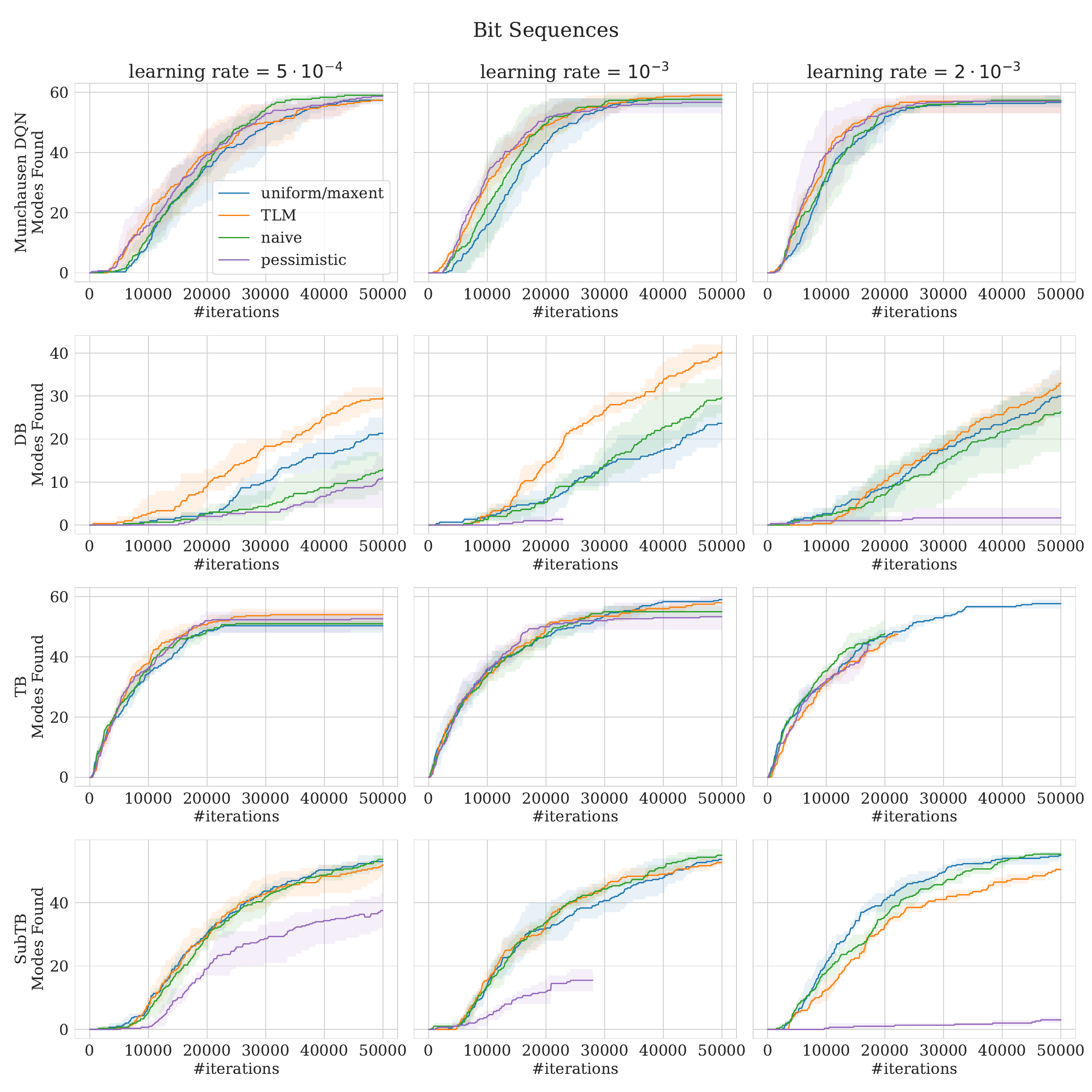}
    \caption{Bit Sequences, the number of modes discovered over the course of training for different methods and a learning rate $\in \{ 5 \cdot 10^{-4}, 10^{-3}, 2 \cdot 10^{-3} \}$. Some results for the learning rates of $10^{-3}$ and $2 \cdot 10^{-3}$ are not full because of exploding gradients at certain points in training.} 
    \label{fig:modes_bitseq_full}
\end{figure}

\begin{figure}[!t]
    \centering
    \includegraphics[width=1\linewidth]{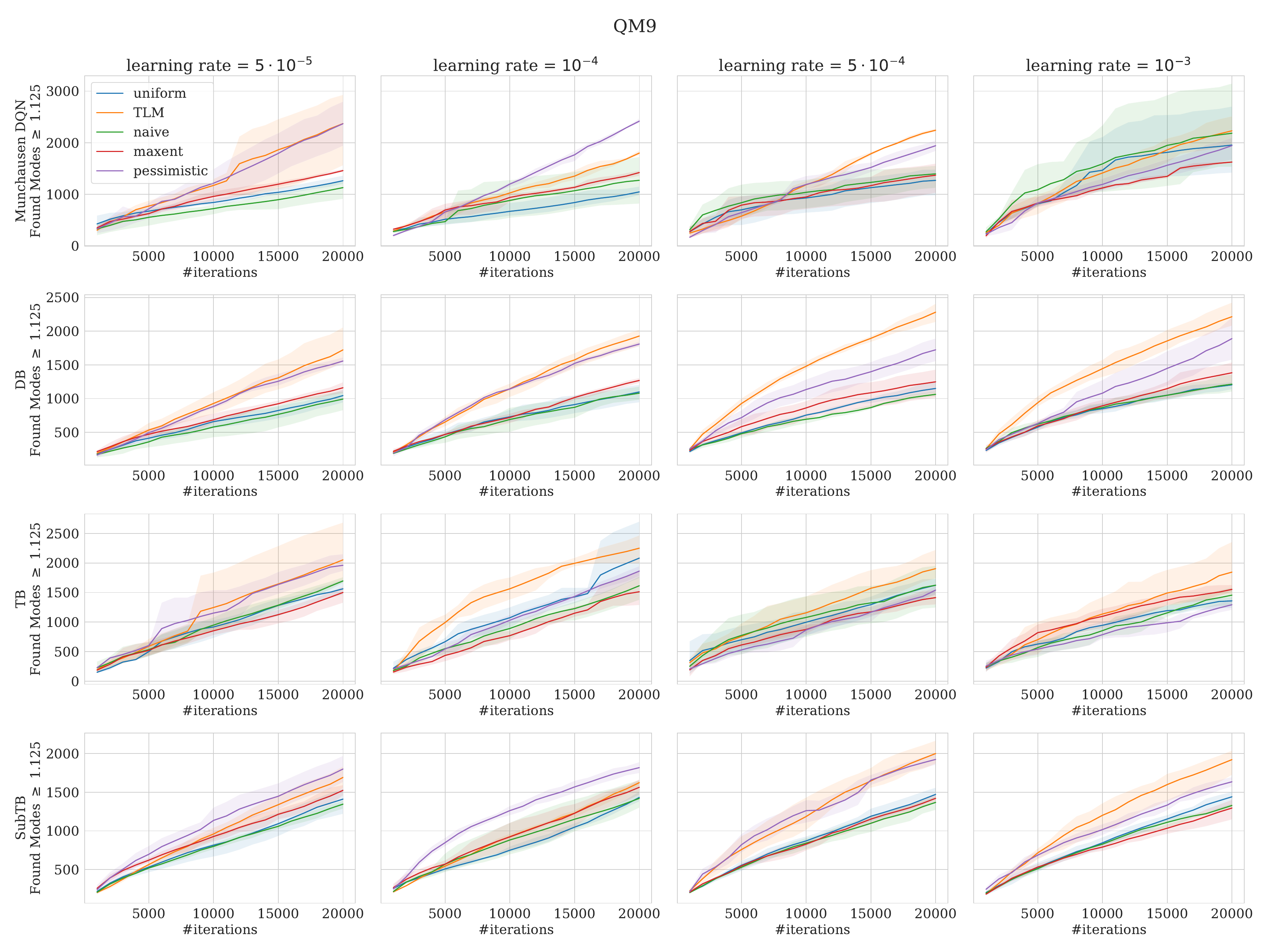}
    \includegraphics[width=1\linewidth]{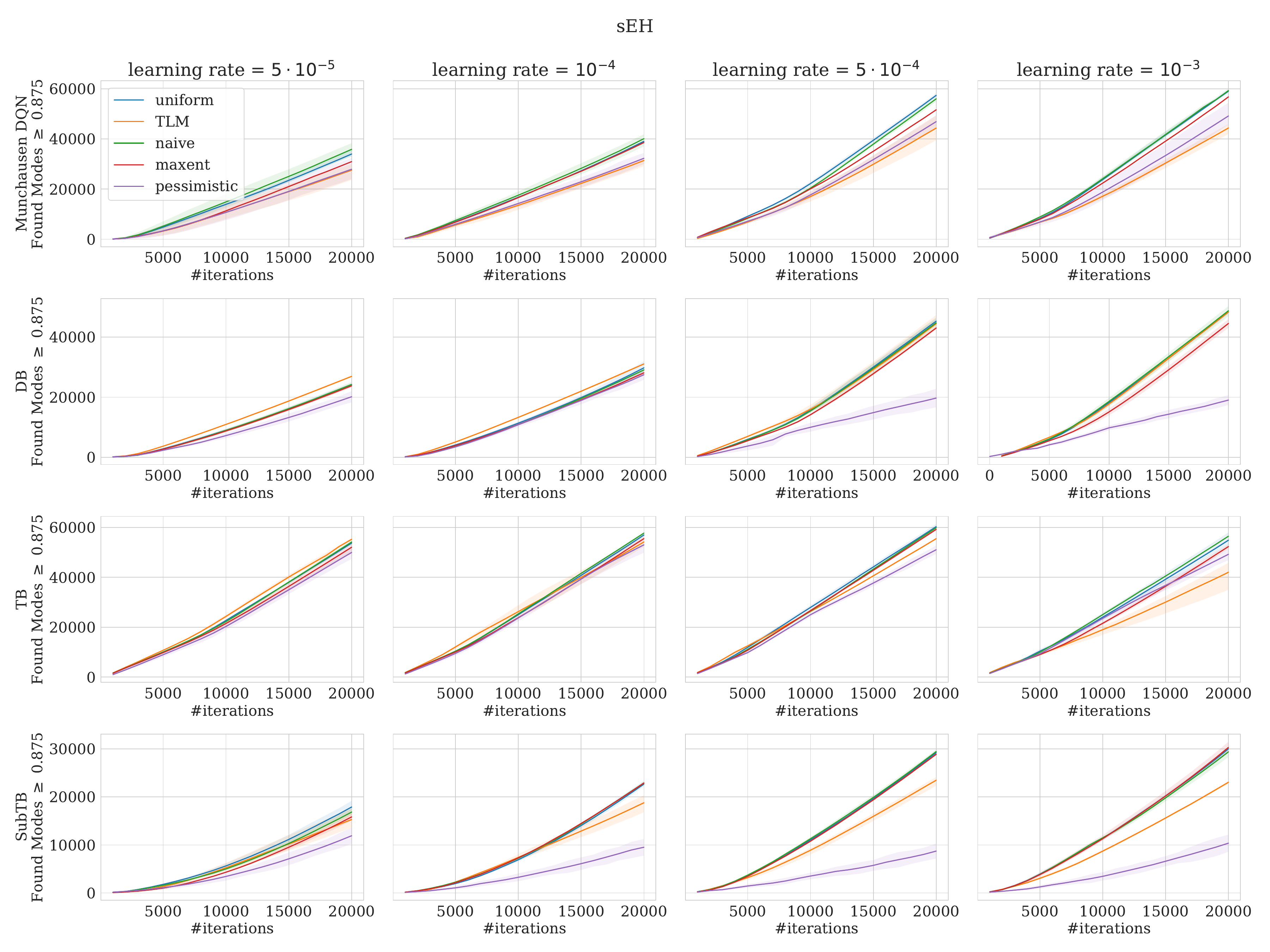}
    \vspace{-0.5cm}
    \caption{\textbf{Top row:} QM9, the number of Tanimoto-separated modes discovered over the course of training with reward higher or equal to $1.125$ for different methods and learning rate in $\in \{ 5 \cdot 10^{-5}, 10^{-4}, 5 \cdot 10^{-4}, 10^{-3} \}$. \textbf{Bottom row:} sEH, the number of Tanimoto-separated modes discovered over the course of training with reward higher or equal to $0.875$ for different methods and learning rate $\in \{ 5 \cdot 10^{-5}, 10^{-4}, 5 \cdot 10^{-4}, 10^{-3} \}$.} 
    \label{fig:modes_mols_full}
\end{figure}

%% file: app/table_mols.tex
\begin{table}[t!]
\begin{center}
\caption{Hyperparameter choices for molecule experiments.}
\label{tab:mol_params}
\vspace{0.3cm}
\begin{tabular}{|l|lr|}
\hline Hyperparatemeter & sEH & QM9  \\
\hline \hline
Reward exponent & \multicolumn{2}{c|}{$10$} \\ 
Batch size & $256$ & $128$ \\
Training iterations & \multicolumn{2}{c|}{$2\times 10^4$} \\ 
Learning rate & 
\multicolumn{2}{c|}{\(\{5 \times 10^{-4}, 10^{-4}, 5 \times 10^{-4}, 10^{-3} \}\)} \\
Learning rate decay $\gamma$ & \multicolumn{2}{c|}{$2^{-1 / 20000}$} \\
$Z$ learning rate for \TB & \multicolumn{2}{c|}{$10^{-3}$}  \\
$Z$ learning rate for \TB decay & \multicolumn{2}{c|}{$2^{-1/50000}$}  \\
Adam optimizer $\beta, \epsilon, \lambda$ & 
\multicolumn{2}{c|}{$(0.9, 0.999), 10^{-8}, 10^{-8}$} \\ 
$\varepsilon$-uniform exploration & \multicolumn{2}{c|}{$0.05$ linearly annealed to $0.0$} \\
Number of transformer layers & \multicolumn{2}{c|}{$8$} \\
Hidden embedding size & \multicolumn{2}{c|}{$256$} \\
\hline
    \SubTB $\lambda$ & \multicolumn{2}{c|}{$1.0$} \\
\hline
\MDQN $\alpha$ & \multicolumn{2}{c|}{$0.15$} \\
\MDQN $\l_0$ & \multicolumn{2}{c|}{$-500$} \\
\MDQN leaf coefficient & \multicolumn{2}{c|}{$10$} \\
\hline
\pessim buffer size & $20 \times 256$ & $20 \times 128$ \\ 
\hline
\TLM backward learning rate & \multicolumn{2}{c|}{$0.1 \times$ learning rate} \\
\TLM backward lr decay $\gamma$ & $2^{-1 / 5000}$ & $2^{-1 / 20000}$ \\ 
\TLM backward target model $\tau$ & \multicolumn{2}{c|}{$0.05$} \\
\hline
\end{tabular}
\end{center}
\end{table}

%% file: main.bbl
\begin{thebibliography}{51}
\providecommand{\natexlab}[1]{#1}
\providecommand{\url}[1]{\texttt{#1}}
\expandafter\ifx\csname urlstyle\endcsname\relax
  \providecommand{\doi}[1]{doi: #1}\else
  \providecommand{\doi}{doi: \begingroup \urlstyle{rm}\Url}\fi

\bibitem[Atanackovic et~al.(2024)Atanackovic, Tong, Wang, Lee, Bengio, and Hartford]{atanackovic2024dyngfn}
Lazar Atanackovic, Alexander Tong, Bo~Wang, Leo~J Lee, Yoshua Bengio, and Jason~S Hartford.
\newblock Dyngfn: Towards bayesian inference of gene regulatory networks with gflownets.
\newblock \emph{Advances in Neural Information Processing Systems}, 36, 2024.

\bibitem[Bengio et~al.(2021)Bengio, Jain, Korablyov, Precup, and Bengio]{bengio2021flow}
Emmanuel Bengio, Moksh Jain, Maksym Korablyov, Doina Precup, and Yoshua Bengio.
\newblock Flow network based generative models for non-iterative diverse candidate generation.
\newblock \emph{Advances in Neural Information Processing Systems}, 34:\penalty0 27381--27394, 2021.

\bibitem[Bengio et~al.(2023)Bengio, Lahlou, Deleu, Hu, Tiwari, and Bengio]{bengio2021gflownet}
Yoshua Bengio, Salem Lahlou, Tristan Deleu, Edward~J. Hu, Mo~Tiwari, and Emmanuel Bengio.
\newblock Gflownet foundations.
\newblock \emph{Journal of Machine Learning Research}, 24\penalty0 (210):\penalty0 1--55, 2023.

\bibitem[Besbes et~al.(2015)Besbes, Gur, and Zeevi]{besbes2015non}
Omar Besbes, Yonatan Gur, and Assaf Zeevi.
\newblock Non-stationary stochastic optimization.
\newblock \emph{Operations research}, 63\penalty0 (5):\penalty0 1227--1244, 2015.

\bibitem[Chen \& Mauch(2023)Chen and Mauch]{chen2023order}
Yihang Chen and Lukas Mauch.
\newblock Order-preserving gflownets.
\newblock In \emph{The Twelfth International Conference on Learning Representations}, 2023.

\bibitem[Deleu et~al.(2024)Deleu, Nouri, Malkin, Precup, and Bengio]{deleu2024discrete}
Tristan Deleu, Padideh Nouri, Nikolay Malkin, Doina Precup, and Yoshua Bengio.
\newblock Discrete probabilistic inference as control in multi-path environments.
\newblock In \emph{The 40th Conference on Uncertainty in Artificial Intelligence}, 2024.

\bibitem[Geist et~al.(2019)Geist, Scherrer, and Pietquin]{geist2019theory}
Matthieu Geist, Bruno Scherrer, and Olivier Pietquin.
\newblock A theory of regularized markov decision processes.
\newblock In \emph{International Conference on Machine Learning}, pp.\  2160--2169. PMLR, 2019.

\bibitem[Haarnoja et~al.(2017)Haarnoja, Tang, Abbeel, and Levine]{haarnoja2017reinforcement}
Tuomas Haarnoja, Haoran Tang, Pieter Abbeel, and Sergey Levine.
\newblock Reinforcement learning with deep energy-based policies.
\newblock In \emph{International conference on machine learning}, pp.\  1352--1361. PMLR, 2017.

\bibitem[Haarnoja et~al.(2018)Haarnoja, Zhou, Abbeel, and Levine]{haarnoja2018soft}
Tuomas Haarnoja, Aurick Zhou, Pieter Abbeel, and Sergey Levine.
\newblock Soft actor-critic: Off-policy maximum entropy deep reinforcement learning with a stochastic actor.
\newblock In Jennifer Dy and Andreas Krause (eds.), \emph{Proceedings of the 35th International Conference on Machine Learning}, volume~80 of \emph{Proceedings of Machine Learning Research}, pp.\  1861--1870. PMLR, 10--15 Jul 2018.

\bibitem[He et~al.(2024{\natexlab{a}})He, Bengio, Cai, and Pan]{he2024rectifying}
Haoran He, Emmanuel Bengio, Qingpeng Cai, and Ling Pan.
\newblock Rectifying reinforcement learning for reward matching.
\newblock \emph{arXiv preprint arXiv:2406.02213}, 2024{\natexlab{a}}.

\bibitem[He et~al.(2024{\natexlab{b}})He, Chang, Xu, and Pan]{he2024looking}
Haoran He, Can Chang, Huazhe Xu, and Ling Pan.
\newblock Looking backward: Retrospective backward synthesis for goal-conditioned gflownets.
\newblock \emph{arXiv preprint arXiv:2406.01150}, 2024{\natexlab{b}}.

\bibitem[Hernandez-Garcia et~al.(2023)Hernandez-Garcia, Duval, Volokhova, Bengio, Sharma, Carrier, Koziarski, and Schmidt]{hernandez2023crystal}
Alex Hernandez-Garcia, Alexandre Duval, Alexandra Volokhova, Yoshua Bengio, Divya Sharma, Pierre~Luc Carrier, Micha{\l} Koziarski, and Victor Schmidt.
\newblock Crystal-gfn: sampling crystals with desirable properties and constraints.
\newblock In \emph{37th Conference on Neural Information Processing Systems (NeurIPS 2023)-AI4MAt workshop}, 2023.

\bibitem[Hu et~al.(2023)Hu, Jain, Elmoznino, Kaddar, Lajoie, Bengio, and Malkin]{hu2023amortizing}
Edward~J Hu, Moksh Jain, Eric Elmoznino, Younesse Kaddar, Guillaume Lajoie, Yoshua Bengio, and Nikolay Malkin.
\newblock Amortizing intractable inference in large language models.
\newblock In \emph{The Twelfth International Conference on Learning Representations}, 2023.

\bibitem[Jain et~al.(2022)Jain, Bengio, Hernandez-Garcia, Rector-Brooks, Dossou, Ekbote, Fu, Zhang, Kilgour, Zhang, et~al.]{jain2022biological}
Moksh Jain, Emmanuel Bengio, Alex Hernandez-Garcia, Jarrid Rector-Brooks, Bonaventure~FP Dossou, Chanakya~Ajit Ekbote, Jie Fu, Tianyu Zhang, Michael Kilgour, Dinghuai Zhang, et~al.
\newblock Biological sequence design with gflownets.
\newblock In \emph{International Conference on Machine Learning}, pp.\  9786--9801. PMLR, 2022.

\bibitem[Jain et~al.(2023)Jain, Raparthy, Hern{\'a}ndez-Garc{\i}a, Rector-Brooks, Bengio, Miret, and Bengio]{jain2023multi}
Moksh Jain, Sharath~Chandra Raparthy, Alex Hern{\'a}ndez-Garc{\i}a, Jarrid Rector-Brooks, Yoshua Bengio, Santiago Miret, and Emmanuel Bengio.
\newblock Multi-objective gflownets.
\newblock In \emph{International Conference on Machine Learning}, pp.\  14631--14653. PMLR, 2023.

\bibitem[Jang et~al.(2024)Jang, Jang, Kim, Park, and Ahn]{jang2024pessimistic}
Hyosoon Jang, Yunhui Jang, Minsu Kim, Jinkyoo Park, and Sungsoo Ahn.
\newblock Pessimistic backward policy for gflownets.
\newblock In \emph{Advances in Neural Information Processing Systems}, volume~37, pp.\  107087--107111. Curran Associates, Inc., 2024.

\bibitem[Jin et~al.(2020)Jin, Barzilay, and Jaakkola]{jin2020junction}
Wengong Jin, Regina Barzilay, and Tommi Jaakkola.
\newblock Junction tree variational autoencoder for molecular graph generation.
\newblock In \emph{Artificial Intelligence in Drug Discovery}, pp.\  228--249. The Royal Society of Chemistry, 2020.

\bibitem[Kim et~al.(2023)Kim, Yun, Bengio, Zhang, Bengio, Ahn, and Park]{kim2023local}
Minsu Kim, Taeyoung Yun, Emmanuel Bengio, Dinghuai Zhang, Yoshua Bengio, Sungsoo Ahn, and Jinkyoo Park.
\newblock Local search gflownets.
\newblock In \emph{The Twelfth International Conference on Learning Representations}, 2023.

\bibitem[Kostenetskiy et~al.(2021)Kostenetskiy, Chulkevich, and Kozyrev]{kostenetskiy2021hpc}
PS~Kostenetskiy, RA~Chulkevich, and VI~Kozyrev.
\newblock Hpc resources of the higher school of economics.
\newblock In \emph{Journal of Physics: Conference Series}, volume 1740, pp.\  012050. IOP Publishing, 2021.

\bibitem[Krichel et~al.(2024)Krichel, Malkin, Lahlou, and Bengio]{krichel2024generalization}
Anas Krichel, Nikolay Malkin, Salem Lahlou, and Yoshua Bengio.
\newblock On generalization for generative flow networks.
\newblock \emph{arXiv preprint arXiv:2407.03105}, 2024.

\bibitem[Lau et~al.(2025)Lau, Lu, Pan, Precup, and Bengio]{lau2024qgfn}
Elaine Lau, Stephen Lu, Ling Pan, Doina Precup, and Emmanuel Bengio.
\newblock Qgfn: Controllable greediness with action values.
\newblock \emph{Advances in Neural Information Processing Systems}, 37:\penalty0 81645--81676, 2025.

\bibitem[Liu et~al.(2024)Liu, Liu, Yang, Xue, Cai, Zhao, Zhang, Hu, Li, and Jiang]{liu2024modeling}
Ziru Liu, Shuchang Liu, Bin Yang, Zhenghai Xue, Qingpeng Cai, Xiangyu Zhao, Zijian Zhang, Lantao Hu, Han Li, and Peng Jiang.
\newblock Modeling user retention through generative flow networks.
\newblock In \emph{Proceedings of the 30th ACM SIGKDD Conference on Knowledge Discovery and Data Mining}, pp.\  5497--5508, 2024.

\bibitem[Madan et~al.(2023)Madan, Rector-Brooks, Korablyov, Bengio, Jain, Nica, Bosc, Bengio, and Malkin]{madan2023learning}
Kanika Madan, Jarrid Rector-Brooks, Maksym Korablyov, Emmanuel Bengio, Moksh Jain, Andrei~Cristian Nica, Tom Bosc, Yoshua Bengio, and Nikolay Malkin.
\newblock Learning gflownets from partial episodes for improved convergence and stability.
\newblock In \emph{International Conference on Machine Learning}, pp.\  23467--23483. PMLR, 2023.

\bibitem[Malkin et~al.(2022)Malkin, Jain, Bengio, Sun, and Bengio]{malkin2022trajectory}
Nikolay Malkin, Moksh Jain, Emmanuel Bengio, Chen Sun, and Yoshua Bengio.
\newblock Trajectory balance: Improved credit assignment in gflownets.
\newblock \emph{Advances in Neural Information Processing Systems}, 35:\penalty0 5955--5967, 2022.

\bibitem[Malkin et~al.(2023)Malkin, Lahlou, Deleu, Ji, Hu, Everett, Zhang, and Bengio]{malkin2022gflownets}
Nikolay Malkin, Salem Lahlou, Tristan Deleu, Xu~Ji, Edward~J Hu, Katie~E Everett, Dinghuai Zhang, and Yoshua Bengio.
\newblock Gflownets and variational inference.
\newblock In \emph{The Eleventh International Conference on Learning Representations}, 2023.

\bibitem[Mnih et~al.(2015)Mnih, Kavukcuoglu, Silver, Rusu, Veness, Bellemare, Graves, Riedmiller, Fidjeland, Ostrovski, et~al.]{mnih2015human}
Volodymyr Mnih, Koray Kavukcuoglu, David Silver, Andrei~A Rusu, Joel Veness, Marc~G Bellemare, Alex Graves, Martin Riedmiller, Andreas~K Fidjeland, Georg Ostrovski, et~al.
\newblock Human-level control through deep reinforcement learning.
\newblock \emph{nature}, 518\penalty0 (7540):\penalty0 529--533, 2015.

\bibitem[Mohammadpour et~al.(2024)Mohammadpour, Bengio, Frejinger, and Bacon]{mohammadpour2024maximum}
Sobhan Mohammadpour, Emmanuel Bengio, Emma Frejinger, and Pierre-Luc Bacon.
\newblock Maximum entropy gflownets with soft q-learning.
\newblock In \emph{International Conference on Artificial Intelligence and Statistics}, pp.\  2593--2601. PMLR, 2024.

\bibitem[Morozov et~al.(2024)Morozov, Tiapkin, Samsonov, Naumov, and Vetrov]{morozov2024improving}
Nikita Morozov, Daniil Tiapkin, Sergey Samsonov, Alexey Naumov, and Dmitry Vetrov.
\newblock Improving gflownets with monte carlo tree search.
\newblock \emph{arXiv preprint arXiv:2406.13655}, 2024.

\bibitem[Nachum et~al.(2017)Nachum, Norouzi, Xu, and Schuurmans]{nachum2017bridging}
Ofir Nachum, Mohammad Norouzi, Kelvin Xu, and Dale Schuurmans.
\newblock Bridging the gap between value and policy based reinforcement learning.
\newblock \emph{Advances in neural information processing systems}, 30, 2017.

\bibitem[Neu et~al.(2017)Neu, Jonsson, and G{\'o}mez]{neu2017unified}
Gergely Neu, Anders Jonsson, and Vicen{\c{c}} G{\'o}mez.
\newblock A unified view of entropy-regularized markov decision processes.
\newblock \emph{arXiv preprint arXiv:1705.07798}, 2017.

\bibitem[Ng et~al.(1999)Ng, Harada, and Russell]{ng1999policy}
Andrew~Y Ng, Daishi Harada, and Stuart Russell.
\newblock Policy invariance under reward transformations: Theory and application to reward shaping.
\newblock In \emph{Icml}, volume~99, pp.\  278--287, 1999.

\bibitem[Paszke et~al.(2019)Paszke, Gross, Massa, Lerer, Bradbury, Chanan, Killeen, Lin, Gimelshein, Antiga, et~al.]{paszke2019pytorch}
Adam Paszke, Sam Gross, Francisco Massa, Adam Lerer, James Bradbury, Gregory Chanan, Trevor Killeen, Zeming Lin, Natalia Gimelshein, Luca Antiga, et~al.
\newblock Pytorch: An imperative style, high-performance deep learning library.
\newblock \emph{Advances in neural information processing systems}, 32, 2019.

\bibitem[Ramakrishnan et~al.(2014)Ramakrishnan, Dral, Rupp, and von Lilienfeld]{ramakrishnan2014qm9}
Raghunathan Ramakrishnan, Pavlo~O. Dral, Matthias Rupp, and O.~Anatole von Lilienfeld.
\newblock Quantum chemistry structures and properties of 134 kilo molecules.
\newblock \emph{Scientific Data}, 1\penalty0 (1):\penalty0 140022, 2014.

\bibitem[Schaul et~al.(2016)Schaul, Quan, Antonoglou, and Silver]{schaul2016prioritized}
Tom Schaul, John Quan, Ioannis Antonoglou, and David Silver.
\newblock Prioritized experience replay.
\newblock In Yoshua Bengio and Yann LeCun (eds.), \emph{4th International Conference on Learning Representations, {ICLR} 2016, San Juan, Puerto Rico, May 2-4, 2016, Conference Track Proceedings}, 2016.

\bibitem[Shen et~al.(2023)Shen, Bengio, Hajiramezanali, Loukas, Cho, and Biancalani]{shen2023towards}
Max~W. Shen, Emmanuel Bengio, Ehsan Hajiramezanali, Andreas Loukas, Kyunghyun Cho, and Tommaso Biancalani.
\newblock Towards understanding and improving gflownet training.
\newblock In \emph{Proceedings of the 40th International Conference on Machine Learning}, ICML'23. JMLR.org, 2023.

\bibitem[Silva et~al.(2024)Silva, da~Silva, Alves, Carvalho, Souza, Kaski, Garg, and Mesquita]{silva2024analyzing}
Tiago Silva, Eliezer de~Souza da~Silva, Rodrigo~Barreto Alves, Luiz~Max Carvalho, Amauri~H Souza, Samuel Kaski, Vikas Garg, and Diego Mesquita.
\newblock Analyzing gflownets: Stability, expressiveness, and assessment.
\newblock In \emph{ICML 2024 Workshop on Structured Probabilistic Inference \& Generative Modeling}, 2024.

\bibitem[Silver et~al.(2014)Silver, Lever, Heess, Degris, Wierstra, and Riedmiller]{silver2014deterministic}
David Silver, Guy Lever, Nicolas Heess, Thomas Degris, Daan Wierstra, and Martin Riedmiller.
\newblock Deterministic policy gradient algorithms.
\newblock In \emph{International conference on machine learning}, pp.\  387--395. Pmlr, 2014.

\bibitem[Sutton \& Barto(2018)Sutton and Barto]{sutton2018reinforcement}
Richard~S Sutton and Andrew~G Barto.
\newblock \emph{Reinforcement learning: An introduction}.
\newblock MIT press, 2018.

\bibitem[Tiapkin et~al.(2023)Tiapkin, Belomestny, Calandriello, Moulines, Munos, Naumov, Perrault, Tang, Valko, and M\'{e}nard]{tiapkin2023fast}
Daniil Tiapkin, Denis Belomestny, Daniele Calandriello, \'{E}ric Moulines, R\'{e}mi Munos, Alexey Naumov, Pierre Perrault, Yunhao Tang, Micha Valko, and Pierre M\'{e}nard.
\newblock Fast rates for maximum entropy exploration.
\newblock In \emph{Proceedings of the 40th International Conference on Machine Learning}, ICML'23. JMLR.org, 2023.

\bibitem[Tiapkin et~al.(2024)Tiapkin, Morozov, Naumov, and Vetrov]{tiapkin2024generative}
Daniil Tiapkin, Nikita Morozov, Alexey Naumov, and Dmitry~P Vetrov.
\newblock Generative flow networks as entropy-regularized rl.
\newblock In \emph{International Conference on Artificial Intelligence and Statistics}, pp.\  4213--4221. PMLR, 2024.

\bibitem[Uehara et~al.(2024)Uehara, Zhao, Biancalani, and Levine]{uehara2024understanding}
Masatoshi Uehara, Yulai Zhao, Tommaso Biancalani, and Sergey Levine.
\newblock Understanding reinforcement learning-based fine-tuning of diffusion models: A tutorial and review.
\newblock \emph{arXiv preprint arXiv:2407.13734}, 2024.

\bibitem[Vaswani et~al.(2017)Vaswani, Shazeer, Parmar, Uszkoreit, Jones, Gomez, Kaiser, and Polosukhin]{vaswani2017attention}
Ashish Vaswani, Noam Shazeer, Niki Parmar, Jakob Uszkoreit, Llion Jones, Aidan~N Gomez, {\L}ukasz Kaiser, and Illia Polosukhin.
\newblock Attention is all you need.
\newblock \emph{Advances in neural information processing systems}, 30, 2017.

\bibitem[Venkatraman et~al.(2024)Venkatraman, Jain, Scimeca, Kim, Sendera, Hasan, Rowe, Mittal, Lemos, Bengio, Adam, Rector-Brooks, Bengio, Berseth, and Malkin]{venkatraman2024amortizing}
Siddarth Venkatraman, Moksh Jain, Luca Scimeca, Minsu Kim, Marcin Sendera, Mohsin Hasan, Luke Rowe, Sarthak Mittal, Pablo Lemos, Emmanuel Bengio, Alexandre Adam, Jarrid Rector-Brooks, Yoshua Bengio, Glen Berseth, and Nikolay Malkin.
\newblock Amortizing intractable inference in diffusion models for vision, language, and control.
\newblock In \emph{Advances in Neural Information Processing Systems}, volume~37, pp.\  76080--76114. Curran Associates, Inc., 2024.

\bibitem[Vieillard et~al.(2020)Vieillard, Pietquin, and Geist]{vieillard2020munchausen}
Nino Vieillard, Olivier Pietquin, and Matthieu Geist.
\newblock Munchausen reinforcement learning.
\newblock \emph{Advances in Neural Information Processing Systems}, 33:\penalty0 4235--4246, 2020.

\bibitem[Zahavy et~al.(2021)Zahavy, O'Donoghue, Desjardins, and Singh]{zahavy2021reward}
Tom Zahavy, Brendan O'Donoghue, Guillaume Desjardins, and Satinder Singh.
\newblock Reward is enough for convex mdps.
\newblock \emph{Advances in Neural Information Processing Systems}, 34:\penalty0 25746--25759, 2021.

\bibitem[Zhang et~al.(2022)Zhang, Malkin, Liu, Volokhova, Courville, and Bengio]{zhang2022generative}
Dinghuai Zhang, Nikolay Malkin, Zhen Liu, Alexandra Volokhova, Aaron Courville, and Yoshua Bengio.
\newblock Generative flow networks for discrete probabilistic modeling.
\newblock In \emph{International Conference on Machine Learning}, pp.\  26412--26428. PMLR, 2022.

\bibitem[Zhang et~al.(2023)Zhang, Dai, Malkin, Courville, Bengio, and Pan]{zhang2023solving}
Dinghuai Zhang, Hanjun Dai, Nikolay Malkin, Aaron~C Courville, Yoshua Bengio, and Ling Pan.
\newblock Let the flows tell: Solving graph combinatorial problems with gflownets.
\newblock In \emph{Advances in Neural Information Processing Systems}, volume~36, pp.\  11952--11969, 2023.

\bibitem[Zhang et~al.(2025)Zhang, Zhang, Gu, ZHANG, Susskind, Jaitly, and Zhai]{zhang2024improving}
Dinghuai Zhang, Yizhe Zhang, Jiatao Gu, Ruixiang ZHANG, Joshua~M. Susskind, Navdeep Jaitly, and Shuangfei Zhai.
\newblock Improving {GF}lownets for text-to-image diffusion alignment.
\newblock \emph{Transactions on Machine Learning Research}, 2025.
\newblock ISSN 2835-8856.

\bibitem[Zhang et~al.(2020)Zhang, Liu, and Xie]{zhang2020molecularmechanicsdrivengraphneural}
Shuo Zhang, Yang Liu, and Lei Xie.
\newblock Molecular mechanics-driven graph neural network with multiplex graph for molecular structures.
\newblock \emph{arXiv preprint arXiv:2011.07457}, 2020.

\bibitem[Zhu et~al.(2024)Zhu, Wu, Hu, Yan, Hou, Wu, et~al.]{zhu2024sample}
Yiheng Zhu, Jialu Wu, Chaowen Hu, Jiahuan Yan, Tingjun Hou, Jian Wu, et~al.
\newblock Sample-efficient multi-objective molecular optimization with gflownets.
\newblock \emph{Advances in Neural Information Processing Systems}, 36, 2024.

\bibitem[Zinkevich(2003)]{zinkevich2003online}
Martin Zinkevich.
\newblock Online convex programming and generalized infinitesimal gradient ascent.
\newblock In \emph{Proceedings of the 20th international conference on machine learning (icml-03)}, pp.\  928--936, 2003.

\end{thebibliography}
